\documentclass[12pt]{article}
\usepackage[draft]{jir-style}
\usepackage[top=1in,bottom=.9in,left=1.2in,right=1.2in]{geometry}
\usepackage{algorithm} 
\usepackage{algorithmic}
\usepackage{graphicx}
\usepackage{subcaption}
\usepackage{booktabs}
\usepackage{tikz}
\usepackage{float}
\usepackage{array}
\usepackage{microtype}
\usepackage{enumitem}

\usepackage{framed}
\providecommand{\keyemph}[1]{\demph{#1}}

\newcommand{\parameterMap}[1]{\Psi_{#1}}
\DeclareMathOperator{\expdim}{expdim}
\DeclareMathOperator{\codim}{codim}
\DeclareMathOperator{\defect}{defect}

\newcommand{\depth}{L}
\newcommand{\dimV}{D}
\newcommand{\RR}{\mathbb{R}}

\newcommand{\calV}{\mathcal{V}}

\newcommand{\bfd}{\mathbf{d}}
\newcommand{\bfs}{\mathbf{s}}
\newcommand{\bfa}{\mathbf{a}}

\newcommand{\ambientSpace}{%
  \operatorname{Sym}_{r^{\depth-1}}(\RR^{d_0})^{\oplus d_\depth}%
}

\newcommand{\leftExpDimension}{%
  \sum_{i=1}^{\depth} d_i d_{i-1}
  - \sum_{i=1}^{\depth-1} d_i%
}

\author{Kevin Dao and Jose Israel Rodriguez}

\title{Minimal Filling Architectures for \\Polynomial Neural Networks}
\date{}

\colorlet{LayerColor1}{blue!30}
\colorlet{LayerColor2}{green!30}
\colorlet{LayerColor3}{yellow!40}
\colorlet{LayerColor4}{orange!40}
\colorlet{LayerColor5}{red!30}
\colorlet{LayerColor6}{purple!30}
\colorlet{LayerColor7}{cyan!30}
\colorlet{LayerColor8}{gray!40}
\colorlet{LayerColor9}{blue!30}
\colorlet{LayerColor10}{green!30}

\newcommand{\DrawNetwork}[3]{%
\begin{tikzpicture}[
  neuron/.style={
    circle,
    draw=black!80,
    minimum size=2mm,
    inner sep=2.2pt
  },
  conn/.style={
    draw=gray!75,
    line width=0.25pt,
    opacity=0.63
  }
]

\def\xsep{#1}
\def\ysep{#2}

\def\LayerSizes{#3}
\def\NumColors{10}

\foreach \n [count=\L] in \LayerSizes {%

  \expandafter\xdef\csname LS\L\endcsname{\n}%

  \xdef\NumLayers{\L}%

  \pgfmathsetmacro{\x}{(\L-1)*\xsep}%
  \pgfmathsetmacro{\offset}{(\n-1)/2}%

  \pgfmathtruncatemacro{\cidx}{mod(\L-1,\NumColors)+1}%

  \foreach \i in {1,...,\n} {%
    \pgfmathsetmacro{\y}{(\i-1-\offset)*\ysep}%
    \node[
      neuron,
      fill=LayerColor\cidx
    ] (N\L\i) at (\x,\y) {};%
  }%
}

\pgfmathtruncatemacro{\LastLayer}{\NumLayers-1}%

\ifnum\NumLayers>1
  \foreach \L in {1,...,\LastLayer} {%
    \pgfmathtruncatemacro{\Lp}{\L+1}%
    \edef\a{\csname LS\L\endcsname}%
    \edef\b{\csname LS\Lp\endcsname}%

    \foreach \i in {1,...,\a} {%
      \foreach \j in {1,...,\b} {%
        \draw[conn] (N\L\i) -- (N\Lp\j);%
      }%
    }%
  }%
\fi

\end{tikzpicture}%
}

\DeclareMathOperator{\Sym}{\operatorname{Sym}}

\usepackage{titlesec} \let\oldappendix\appendix \renewcommand{\appendix}{% 
\oldappendix \titleformat{\section} {\normalfont\Large\bfseries} {Appendix~\thesection: } {0em} {}% 
}
\crefname{appendix}{Appendix}{Appendices} 
\Crefname{appendix}{Appendix}{Appendices} 
\crefalias{section}{appendix}

\begin{document}
\maketitle

\begin{abstract} 
Minimal filling architectures describe the smallest hidden-layer widths for which the neurovariety of a polynomial neural network fills its ambient function space. It was conjectured that every such architecture has unimodal hidden widths. We construct counterexamples to this conjecture using a frontier search and certify the main example through finite-field Jacobian computations and recursive dimension bounds. In the opposite direction, we prove that every minimal filling architecture of depth at most three is unimodal. A nonunimodal minimal filling architecture of depth four shows that the depth bound is sharp. 
\end{abstract}

\begin{figure*}[t]
  \centering
  \begin{tikzpicture}[scale=.9, %we can change scale and minimum size to make this smaller if need be
% originally scale =1 and minimum size = 6mm
% scale=0.78 is the largest I can scale to fit the last paragraph of 4.3 in, but it does seem to shift/add more white space in some parts. Feel free to change.
  neuron/.style={circle, draw, minimum size=6mm, inner sep=0pt},
  connection/.style={draw, opacity=0.35}
]

% ---------- Layer 1 (2) ----------
\foreach \i in {1,...,2}
  \node[neuron, fill=blue!30] (L1-\i)
  at (0, {(2+1)/2 - \i}) {};

% ---------- Layer 2 (3) ----------
\foreach \i in {1,...,3}
  \node[neuron, fill=green!30] (L2-\i)
  at (2, {(3+1)/2 - \i}) {};

% ---------- Layer 3 (4) ----------
\foreach \i in {1,...,4}
  \node[neuron, fill=yellow!40] (L3-\i)
  at (4, {(4+1)/2 - \i}) {};

% ---------- Layer 4 (5) ----------
\foreach \i in {1,...,5}
  \node[neuron, fill=orange!40] (L4-\i)
  at (6, {(5+1)/2 - \i}) {};

% ---------- Layer 5 (4) ----------
\foreach \i in {1,...,4}
  \node[neuron, fill=red!30] (L5-\i)
  at (8, {(4+1)/2 - \i}) {};

% ---------- Layer 6 (6) ----------
\foreach \i in {1,...,6}
  \node[neuron, fill=purple!30] (L6-\i)
  at (10, {(6+1)/2 - \i}) {};

% ---------- Layer 7 (4) ----------
\foreach \i in {1,...,4}
  \node[neuron, fill=cyan!30] (L7-\i)
  at (12, {(4+1)/2 - \i}) {};

% ---------- Layer 8 (1) ----------
\node[neuron, fill=gray!40] (L8-1) at (14, 0) {};

% ---------- Connections ----------
\foreach \i in {1,...,2}
  \foreach \j in {1,...,3}
    \draw[connection] (L1-\i) -- (L2-\j);

\foreach \i in {1,...,3}
  \foreach \j in {1,...,4}
    \draw[connection] (L2-\i) -- (L3-\j);

\foreach \i in {1,...,4}
  \foreach \j in {1,...,5}
    \draw[connection] (L3-\i) -- (L4-\j);

\foreach \i in {1,...,5}
  \foreach \j in {1,...,4}
    \draw[connection] (L4-\i) -- (L5-\j);

\foreach \i in {1,...,4}
  \foreach \j in {1,...,6}
    \draw[connection] (L5-\i) -- (L6-\j);

\foreach \i in {1,...,6}
  \foreach \j in {1,...,4}
    \draw[connection] (L6-\i) -- (L7-\j);

\foreach \i in {1,...,4}
  \draw[connection] (L7-\i) -- (L8-1);

\end{tikzpicture}
  \caption{Our first counterexample to the unimodal MFA conjecture ($r=2$)} 
  \label{fig:wide}
\end{figure*}

\section{Polynomial Neural Networks and the Unimodality Conjecture}

Neuroalgebraic geometry studies the geometry of function spaces parameterized by machine-learning models when they are algebraic or semi-algebraic in nature. 
It relates algebro-geometric invariants (e.g., dimension and singularities) to notions such as expressivity and training behavior \cite{Marchettietal2025}. 
This viewpoint involves other geometry-driven approaches in learning theory, including tropical-geometric descriptions of ReLU networks \cite{pmlr-v80-zhang18i-tropical-neu-ag,brandenburg2024-tropical-neu-ag}.
In this article, we study polynomial neural networks. Related neuroalgebraic questions concerning model dimension and identifiability have also been investigated for other polynomially parameterized architectures, including lightning self-attention \cite{ICLR2025_259e59fe-neu-ag}.

We state the {unimodal minimal filling architecture} conjecture 
in its original form, Conjecture 12 of \cite{KTB}, and recall the necessary definitions.

\begin{definition}
    A (feedforward) \demph{polynomial neural network} (PNN) with activation function $\sigma$ 
    is a function
    of the form
    \[
    p_\theta:\mathbb{R}^{d_0}\to \mathbb{R}^{d_\depth},\quad 
p_\theta(x)=W_\depth\sigma W_{\depth-1}\sigma \dots \sigma W_1 x
    \]
    where $\sigma$ acts entrywise by $(\sigma(z))_j = z_j^r$ for some $r\in\mathbb{N}$.
    The parameters  are $\theta=(W_1,\dots,W_\depth)$, with each \demph{weight matrix}
$W_i\in\mathbb{R}^{d_i\times d_{i-1}}$.
    The network's \demph{architecture} is the sequence $\bfd=(d_0,\dots,d_\depth)$, and the network's
    \demph{depth} is $\depth$.
\end{definition}

\begin{remark}
In this article, we use the terminology ``feedforward polynomial neural network"
to be consistent with \cite{KLW,KTB}. In the literature, 
these are also referred to as ``monomial multilayer perceptrons (MLPs)". The MLP terminology is useful when distinguishing fully connected feedforward neural networks from convolutional neural networks. 
Moreover, prefacing MLP with monomial {clarifies} when the activation function is monomial, as opposed to a general polynomial in a single variable. 
\end{remark}

\begin{definition}
    For fixed $r$ and architecture $\bfd$, the  map
    $$\parameterMap{\bfd,r}:
    \mathbb{R}^{{d_1\times d_0}}\times \cdots \times \mathbb{R}^{{d_\depth}\times d_{\depth-1}}
    \to \ambientSpace$$ given by $\theta\mapsto \begin{bmatrix}
        p_{\theta 1} & \dots & p_{\theta d_{\depth}}
    \end{bmatrix}^T$ where each $p_{\theta i}$ is a polynomial in the entries of the weight matrices is called
    the \demph{parameter map}. The codomain is referred to as the \demph{ambient space}, and its dimension is well known:
    \[
\dim\left({\ambientSpace}\right)=d_\depth\binom{d_0+r^{\depth-1}-1}{r^{\depth-1}}.
    \]
    The image of $\parameterMap{\bfd,r}$ is a set of vectors of polynomials denoted $\mathcal{F}_{\bfd,r}$ which we call the \demph{functional space}. Its Zariski closure $\calV_{\bfd,r}=\overline{\mathcal{F}_{\bfd,r}}$ is called the \demph{neurovariety} as it is an irreducible algebraic variety. To simplify notation, we remove $r$ from the subscript when the context is clear.
\end{definition}

\begin{remark}
The functional space is also called the \demph{hypothesis class} and, in parts of the literature, the \demph{neuromanifold}. 
The term \demph{neural model} is also used, although conventions differ on whether it refers to the parameterized function space or its Zariski closure. 
\end{remark}

\begin{example}
Let $\bfd=(2,d_1,1)$, and denote the input variable into the network by ${x}=\begin{bmatrix}
    x_1 & x_2
\end{bmatrix}^T$. 
If the weights are 
\begin{align*}
    W_1 =\begin{bmatrix}
        1 & 1\\
        1 & 2\\
        \vdots&\vdots \\
        1& d_1
    \end{bmatrix}
     \quad \text{and}\quad W_2=\begin{bmatrix}
        1 & 1 &\dots &1
    \end{bmatrix},
\end{align*}
then
 $\Psi_{(2,d_1,1),r}(W_1,W_2)$
 is the polynomial function 
\[
p_{W_1,W_2}(x)=  (x_1+x_2)^r+
(x_1+2x_2)^r+
\cdots+ (x_1+d_1x_2)^r
\]
in ${\Sym_{r}(\mathbb{R}^{2})}$, i.e., 
an order-$r$ symmetric tensor in $(\mathbb{R}^2)^{\otimes r}$.
\end{example}

A goal of neuroalgebraic geometry is to use algebraic geometry to study the function spaces associated with polynomial neural networks. The dimension of a neurovariety provides an algebraic measure of expressivity and has been studied in, for example, \cite{FRWY,KTB,KLW,Massarenti2025,UBDC2025}. Previous work has largely fixed the architecture and examined how the dimension or defectivity of its neurovariety varies with the activation degree $r$. Here we instead fix $r$ and vary the hidden layer widths. We are particularly interested in the extremal case in which the neurovariety equals its ambient space, which motivates the following definition.

\begin{definition}\label{defn:neurovariety}

The neurovariety $\calV_{\bfd,r}$ is said to be \demph{filling} if $\calV_{\bfd,r}=\ambientSpace$. If $r$ is fixed, we will say that the architecture $\bfd$ is filling if its neurovariety is filling.

We define a partial ordering on architectures of the same depth $\depth$ by declaring that $\bfd'\leq \bfd$ if $d_i'\leq d_i$ for $0\leq i\leq \depth$. If $\bfd'< \bfd$, we refer to $\bfd'$ as a \demph{subarchitecture} of $\bfd$.
\end{definition}

Proposition 5 of
\cite{KTB} shows that the fillingness of a neurovariety 
$\calV_{(d_0,\bfa,d_\depth),r}$ implies
surjectivity of the parameter map 
$\Psi_{(d_0,2\bfa,d_\depth),r}$,
which has doubled the width of the hidden layers.

\begin{example}
    According to \cite[Corollary 3.11]{KLW}, the architecture $(2,2,1)$ with activation degree $3$ has a filling neurovariety and a nonsurjective parameter map. In other words, its functional space $\mathcal{F}_{(2,2,1),3}$ is not the entire ambient space.
\end{example}

\begin{definition} For fixed $r$, an architecture $\bfd=(d_0,\dots,d_\depth)$ is
\demph{minimal filling} if $\bfd$ is filling and every subarchitecture
$\bfd'=(d_0,d_1',\dots,d_{\depth-1}',d_\depth)$ with $\bfd'<\bfd$
is not filling. We say it is a \demph{minimal filling architecture (MFA)}. 
\end{definition}

The output of a depth-$\depth$ PNN is a homogeneous polynomial map $p_\theta:\RR^{d_0}\to\RR^{d_\depth}$ of degree $D=r^{\depth-1}$. Under the standard identification of degree $D$ homogeneous polynomials with order $D$ symmetric tensors, $p_\theta$ determines an element of $\Sym_D(\RR^{d_0})^{\oplus d_\depth}.$ 
Layer-by-layer, the architecture $\bfd$ gives a recursive parameterization of a subset of $\Sym_D(\RR^{d_0})^{\oplus d_\depth}$. 
From this perspective, a minimal filling architecture records the minimal intermediate widths needed for the associated neurovariety to fill the ambient space.

\medskip

For fixed $d_0,d_{\depth},r$, and $\depth$, Dickson's lemma, which we recall in~\Cref{lemma:dickson} and is proved in \cite[Chapter 2.4]{CLO-2025}, implies that there are only finitely many minimal filling architectures. 
We are ready to state the {unimodal minimal filling architecture conjecture} 
{of \cite[Conjecture 12]{KTB}.}

\begin{conjecture}\label{conj:ktb}
    Fix $r,\depth,d_0,d_\depth$. If $\bfd=(d_0,d_1,\dots,d_\depth)$ is 
    a MFA, then there is an $i$ with $0\leq i\leq \depth$ such that 
    \[
    d_0\leq d_1\leq \dots\leq d_i\qquad\&\qquad d_i\geq d_{i+1}\geq \dots \geq d_\depth.
    \]
    That is, the sequence $\bfd$ is \keyemph{unimodal}.
\end{conjecture}

There were several reasons to expect the conjecture to hold. It holds for shallow networks by a dimension count and for activation degree $r=1$ by \cite[Section 2.1]{KTB}. Moreover, all previously computed minimal filling architectures were unimodal. 
The conjecture also conforms with the common expand-then-contract heuristic for neural architectures where hidden widths first increase to build expressive features and then decrease toward the output.

\section{The Counterexample: $(2,3,4,5,4,6,4,1)$
}\label{s:counterexample}

Fix $r=2$. Our main contribution addresses {\Cref{conj:ktb}}.

\begin{theorem}\label{theorem:counterexample}
The architecture $\bfd=(2,3,4,5,4,6,4,1)$ is a counterexample to the unimodal minimal filling architecture conjecture. That is, it is \keyemph{minimal filling}, and \keyemph{not unimodal}. 
\end{theorem}

The architecture $\bfd=(2,3,4,5,4,6,4,1)$ is clearly not unimodal. We are tasked with showing it is minimal filling.
{Our approach is to prove (i) the neurovariety $\calV_{\bfd,2}$ is filling 
in \Cref{lemma:counterfill}
and (ii) no subarchitecture is filling in \Cref{lemma:combined}. 
Our proofs involve computing and upper-bounding the dimension of neurovarieties.
}

\begin{lemma}\label{lemma:counterfill}
    For $\bfd$ 
    as in \Cref{theorem:counterexample}, $\calV_{\bfd,2}=\Sym_{2^6}(\mathbb{R}^2)\cong \mathbb{R}^{65}$.
\end{lemma}

\begin{proof} 
It suffices to show $\dim \calV_{\bfd,2}=65$, the dimension of the ambient space. The finite field Jacobian criterion used below is reviewed in \Cref{appendix:A}, and the explicit parameter point and prime are recorded in \Cref{appendix:parameterchoices}.
We consider the Jacobian of the parameter map $\operatorname{Jac}\parameterMap{\bfd,r}$. 
By  \Cref{appendix:theorem1}, 
\[
\operatorname{rank}(\operatorname{Jac}\parameterMap{\bfd,r}(\theta))\leq\dim \calV_{\bfd,r}
\] 
for every 
$\theta\in \RR^{d_1\times d_0}\times\dots\times \RR^{{d_\depth}\times d_{\depth-1}}$. 
By  \Cref{appendix:corollary-integer-coefficients}
and the fact that $\parameterMap{\bfd,r}$ is given by polynomials with integer coefficients, 
\begin{align*}
\operatorname{rank}(\operatorname{Jac}\Psi_{\bfd,r}(\theta)\bmod p) & \leq \dim \calV_{\bfd,r}\leq\dim \ambientSpace
\end{align*}
where $p$ is a prime and $\theta\in \mathbb{Z}^{d_1\times d_0}\times\dots\times \mathbb{Z}^{d_{\depth}\times d_{\depth-1}}$. 
Using the backpropagation algorithm as implemented in \cite{KLW} we found
 a prime $p$ and $\theta$ for which 
\[
\operatorname{rank}(\operatorname{Jac}\parameterMap{\bfd,r}(\theta)\bmod p)=65.
\]
{Our choice of parameters and prime are in \Cref{appendix:parameterchoices}.}
Since the ambient space also has dimension $65$, we conclude $\dim\calV_{\bfd,r}=65$.
\end{proof}

To show minimality, we utilize the following two lemmas.

\begin{lemma}\cite{KLW}\label{lem:MFA-closed-under-increased-hidden-widths}
{Fix $r\geq 1$.}
Let $\bfd=(d_0,\dots,d_\depth)$ and $\bfd'=(d_0,\dots,d_i',\dots,d_\depth)$ be two architectures except they differ in exactly the $i$th entry for $0<i<\depth$. Assume that $d_i'\leq d_i$. Then,
\[
\mathcal{F}_{\bfd'}\subseteq\mathcal{F}_{\bfd}\qquad\&\qquad \calV_{\bfd'}\subseteq \calV_{\bfd}.
\]
In particular, $\dim(\calV_{\bfd'})\leq \dim \left(\calV_{\bfd}\right)$. 
\end{lemma}

By induction, one can generalize \Cref{lem:MFA-closed-under-increased-hidden-widths} in the obvious manner: if $\bfd'\leq \bfd$ with $d_0'=d_0$ and $d_\depth'=d_\depth$, then $\dim\left(\calV_{\bfd'}\right)\leq \dim \left(\calV_{\bfd}\right)$.

\begin{lemma}\label{Lemma:Recursive}\cite{KTB} Fix $r\geq 1$ and let $k\in\{1,\dots,\depth-1\}$. Then,
    \[
    \dim \calV_{(d_0,\dots,d_\depth)}\leq \dim\calV_{(d_0,\dots,d_k)}+\dim\calV_{(d_k,\dots,d_\depth)}-d_k.
    \]
\end{lemma}

By \Cref{lem:MFA-closed-under-increased-hidden-widths}, it suffices to show that the maximal proper subarchitectures obtained from $\bfd$ by decreasing a single hidden-layer width by one are non-filling. 
These subarchitectures, denoted $\bfs_1,\ldots,\bfs_6$, 
are listed in \Cref{table:maximal-non-filling-depth-7}. 
The next lemma establishes the dimension upper bounds shown in the third column of the table. We also report in the table the finite-field Jacobian ranks returned by \Cref{alg:randomized-image-dimension} that provide lower~bounds.

\begin{table}[t]
  \caption{
   For the dimension of the neurovariety
   $\calV_{\bfs_i,2}$ we have
  upper bounds proved in \Cref{lemma:combined} and lower bounds obtained using \Cref{alg:randomized-image-dimension}
  }
    \vskip -0.1in
  \label{table:maximal-non-filling-depth-7}
   \begin{center}
    \begin{small}
      \begin{sc}
        \begin{tabular}{lcccr}
          \toprule
          % Architecture  & 
          % Reported & Proven \\
          %  & Dimension & Bound\\
    Architecture  & Dimension Lower Bounds & Dimension Upper Bounds  \\
                  & (Using \Cref{alg:randomized-image-dimension}) & (Using \Cref{lemma:combined})\\
          \midrule 
          $\mathbf{s}_1=(2,{\color{red}{2}},4,5,4,6,4,1)$    & $35\leq$ & $\leq 53$ \\
          $\mathbf{s}_2=(2,3,{\color{red}{3}},5,4,6,4,1)$ & $60\leq$ & $\leq 61$\\
          $\mathbf{s}_3=(2,3,4,{\color{red}{4}},4,6,4,1)$ & $62\leq$ & $\leq 63$ \\
          $\mathbf{s}_4=(2,3,4,5,{\color{red}{3}},6,4,1)$ & $39\leq$ & $=39$\\
          $\mathbf{s}_5=(2,3,4,5,4,{\color{red}{5}},4,1)$ & $61\leq$ & $\leq 62$ \\
          $\mathbf{s}_6=(2,3,4,5,4,6,{\color{red}{3}},1)$ & $59\leq$ & $\leq 62$\\
          \bottomrule
        \end{tabular}
      \end{sc}
    \end{small}
  \end{center}
  \vskip -0.1in
\end{table}

\begin{lemma}\label{lemma:combined}
The ambient space has dimension $65$. The maximal proper
subarchitectures $\bfs_i$, for $1\leq i\leq6$, listed in
\Cref{table:maximal-non-filling-depth-7} satisfy the dimension bounds
in the third column of the table. In particular, none is filling.
\end{lemma}

\begin{proof}
Recall that $r=2$. For an architecture $\bfa$, write
\[
\dimV_{\bfa}:=\dim\calV_{\bfa,2}.
\]
We use the recursive dimension bound (\Cref{Lemma:Recursive}), the ambient dimension bound, and
the equality $\dimV_{(m,n)}=mn$.

Repeated application of \Cref{Lemma:Recursive}, followed by the ambient-dimension bound, gives

\begin{align*} 
\text{For }\bfs_1:\qquad \dimV_{\bfs_1} &\leq \dimV_{(2,2)} +\dimV_{(2,4,5,4)} +\dimV_{(4,6,4,1)} -2-4\\
&\leq 4+4\dim\Sym_4(\RR^2)+\dim\Sym_4(\RR^4)-6\\
&= 4+20+35-6 = 53, \\[1.25ex] 
\text{For }\bfs_2:\qquad \dimV_{\bfs_2} &\leq \dimV_{(2,3,3)} +\dimV_{(3,5,4)} +\dimV_{(4,6,4,1)} -3-4\\
&\leq 3\dim\Sym_2(\RR^2) +4\dim\Sym_2(\RR^3) +\dim\Sym_4(\RR^4)-7\\ &= 9+24+35-7 = 61, 
\end{align*}
\begin{align*}
\text{For }\bfs_3:\qquad \dimV_{\bfs_3} &\leq \dimV_{(2,3,4,4)} +\dimV_{(4,4)} +\dimV_{(4,6,4,1)} -4-4
\\ &\leq 4\dim\Sym_4(\RR^2) +16 +\dim\Sym_4(\RR^4)-8
\\ &= 20+16+35-8 = 63, \\[1.25ex] 
\text{For }\bfs_4:\qquad \dimV_{\bfs_4} 
&\leq \dimV_{(2,3,4,5,3)} +\dimV_{(3,6,4,1)} -3
\\ 
&\leq 3\dim\Sym_8(\RR^2) +\dim\Sym_4(\RR^3)-3
\\ 
&= 27+15-3 = 39, 
\\[1.25ex] 
\text{For }\bfs_5:\qquad \dimV_{\bfs_5} &\leq \dimV_{(2,3,4,5,4)} +\dimV_{(4,5)} +\dimV_{(5,4,1)} -4-5\\ &\leq 4\dim\Sym_8(\RR^2) +20 +\dim\Sym_2(\RR^5)-9\\ &= 36+20+15-9 = 62, \\[1.25ex] 
\text{For }\bfs_6:\qquad \dimV_{\bfs_6} &\leq \dimV_{(2,3,4,5,4)} +\dimV_{(4,6,3)} +\dimV_{(3,1)} -4-3\\ &\leq 4\dim\Sym_8(\RR^2) +3\dim\Sym_2(\RR^4) +3-7\\ &= 36+30+3-7 = 62.
\end{align*}

Each maximal proper subarchitecture has dimension less
than $\dim \Sym_{r^{\depth-1}}(\RR^2)
 =65$. Hence none of
$\bfs_1,\ldots,\bfs_6$ is filling.
\end{proof}

\begin{proof}[Proof of \Cref{theorem:counterexample}] 
Fix the activation degree $r=2$.
By \Cref{lemma:counterfill}, 
the architecture 
$\bfd=(2,3,4,5,4,6,4,1)$ 
is filling. 
By \Cref{lemma:combined}, 
each maximal proper subarchitecture obtained by decreasing one hidden width by one is non-filling. 
\Cref{lem:MFA-closed-under-increased-hidden-widths} implies that every proper subarchitecture of 
$\bfd$ is non-filling. 
Thus $\bfd$ is minimal filling and it is not unimodal. 
\end{proof}

\section{Unimodality in small depth}

In this section, we show that  \Cref{conj:ktb} does hold if one imposes the condition that $L\leq 3$. 
We also give the nonunimodal minimal filling architecture in \Cref{ex:depth-four-counterexample}, showing that the depth bound is sharp.

\begin{lemma}\label{lemma:d0-bound}
If $L\geq 2$, $r\geq 2$, and $\bfd=(d_0,\dots,d_L)$ is a filling architecture, then $d_1\geq d_0$ and $d_{L-1}\geq \min\left\{
d_L,\binom{d_0+r^{L-1}-1}{r^{L-1}}
\right\}$.
\end{lemma}

\begin{proof}
Let $\ell_1,\dots,\ell_{d_1}$ be the linear forms produced by the first layer $W_1$, and set
\[
U:=\operatorname{Span}_{1\leq i\leq d_1}(\ell_i)
\] 
as their linear span in the vector space of linear forms in variables $x_1,\dots,x_{d_0}$. 

At hidden layer $k$ the polynomials produced by the network satisfy
\[
W_k\sigma W_{k-1}\sigma\dots \sigma W_1 x\in\Sym_{r^{k-1}}(U).
\]
Let $D=r^{\depth-1}$.
Suppose that $d_1<d_0$. 
Although the subspace $U$ varies with $W_1$, 
the image of the parameter map is contained in the subspace variety 
\[ 
\operatorname{Sub}_{d_1} := 
\left\{ \mathbf p\in \Sym_D(\RR^{d_0})^{\oplus d_\depth} : \mathbf p\in\Sym_D(U)^{\oplus d_\depth} \text{ for some \(U\subseteq(\RR^{d_0})\) with } \dim U\leq  d_1 \right\}. 
\] 
The incidence variety parameterizing a $d_1$-dimensional subspace $U$ together with a tuple in 
$\Sym_D(U)^{\oplus d_L}$ 
has dimension at most
\[ d_1(d_0-d_1) + d_L\binom{d_1+D-1}{D}. \]
Since $D\geq2$ and $d_1<d_0$, we have 
\[ d_1(d_0-d_1) + d_L\binom{d_1+D-1}{D} < d_L\binom{d_0+D-1}{D}. \] 
Therefore, $\operatorname{Sub}_{d_1}$, and hence the neurovariety, has dimension strictly smaller than the ambient space. 
Thus a filling architecture must have $d_1\geq d_0$.    

\smallskip

Next, we show $d_{\depth-1}\geq \min\left\{d_\depth,\binom{d_0+r^{\depth-1}-1}{r^{\depth-1}}\right\}$.  For convenience, here we set 
\[ 
N:=\dim\Sym_{r^{L-1}}(\mathbb R^{d_0}) =\binom{d_0+r^{L-1}-1}{r^{L-1}}. 
\]
Start by choosing a monomial basis for $\Sym_{r^{\depth-1}}(\mathbb{R}^{d_0})$. Let $\mathcal{C}$ be the $d_\depth\times N$-matrix 
of the coefficients of the $d_\depth$ many output polynomials. 
This matrix factors as $\mathcal{C}=W_\depth M$ where $M$ is some  $d_{\depth-1}\times N$-matrix. 
This implies $\mathcal{C}$ has rank at most $d_{\depth-1}$. If $d_{\depth-1}$ were smaller than both $d_\depth$ and $N$, then $\mathcal{C}$ would lie on a proper determinantal variety and so the neurovariety cannot be filling. Thus, we must have $d_{\depth-1}\geq \min\left\{d_\depth,N\right\}$.
\end{proof}

\begin{corollary}\label{corr:shallowmfas}
Fix $r\geq 2$, $d_0$, and $d_2$. 
If $\bfd=(d_0,d_1,d_2)$ is an MFA, then $\bfd$ is unimodal.
\end{corollary}

\begin{lemma}\label{lemma:d1-bound}
    If $\depth\geq 2$ and $\bfd=(d_0,\dots,d_\depth)$ is an MFA, then $d_1\leq \binom{d_0+r-1}{r}$.
\end{lemma}

\begin{proof}
Set
$N:=\dim\Sym_r(\RR^{d_0})
=\binom{d_0+r-1}{r}.$
It suffices to show that  $\bfd=(d_0,d_1,d_2,\dots,d_\depth)$ being filling, implies 
$\bfd'=(d_0,N,d_2,\dots,d_\depth)$ is filling. 

Choose linear forms $m_1,\ldots,m_N$ such that 
\[ m_1^r,\ldots,m_N^r \] 
form a basis of $\Sym_r(\RR^{d_0})$. 
For any choice of weights in the architecture $\bfd$, each component of 
\[ W_2\sigma(W_1x) \] 
is a homogeneous polynomial of degree $r$ and hence is a linear combination of $m_1^r,\ldots,m_N^r$. 
Let $W_1'$ be the matrix whose rows are the coefficient vectors of $m_1,\ldots,m_N$, 
and let $W_2'$ be the matrix of the corresponding expansion coefficients. 
Then 
\[ W_2'\sigma(W_1'x)=W_2\sigma(W_1x). \]
Keeping all subsequent weight matrices unchanged shows that every function realized by $\bfd$ is also realized by $\bfd'$ 
Therefore, if $\bfd$ is filling, then $\bfd'$ is filling.
Since $\bfd=(d_0,d_1,\dots,d_\depth)$ is a MFA, we have
\[
d_1\leq N
=
\binom{d_0+r-1}{r}.\vspace{-12pt}
\]
\end{proof}

The unimodal MFA conjecture is true for networks of depth at most three.

\begin{theorem}\label{theorem:uni}
    Let $\depth \leq 3$. Fix $r\geq 1$, $d_0$, and $d_\depth$. If $\bfd=(d_0,\dots,d_\depth)$ is an MFA, then $\bfd$ is unimodal. 
\end{theorem}

\begin{proof}

The case $r=1$ is \cite[Section 2.1]{KTB}, so let us assume $r\geq 2$. Consider the various choices of $\depth$. For $\depth=1$, any architecture has only two entries and is automatically unimodal. For the shallow case, $\depth=2$, this is \Cref{corr:shallowmfas}.

For $\depth=3$, any MFA $\bfd=(d_0,d_1,d_2,d_3)$ has 
\[
d_0\leq d_1\leq \binom{d_0+r-1}{r}\qquad\&\qquad 
d_2\geq\min
    \left\{
    d_3,\,\binom{d_0+r^2-1}{r^2}\right\}
\]
by \Cref{lemma:d0-bound,lemma:d1-bound}. Toward a contradiction, assume $\bfd$ is not unimodal. Then, $d_1>d_2<d_3$. Since $d_2<d_3$, it follows $d_3>\binom{d_0+r^2-1}{r^2}$ and hence $d_2\geq \binom{d_0+r^2-1}{r^2}$. On the other hand, because $r\leq r^2$, we get 
 \[ d_1 \leq \binom{d_0+r-1}{r} \leq \binom{d_0+r^2-1}{r^2} \leq d_2, \]
which is a contradiction.
\end{proof}

\begin{example}[$\depth=4$ counterexample]\label{ex:depth-four-counterexample}
Let $r=6$ and consider the depth $\depth=4$ architecture $\bfd=(2,6,5,38,1).$ Applying \Cref{alg:randomized-image-dimension} shows that $\bfd$ is filling, as it returns the ambient dimension 
\[ 
\dim\Sym_{6^3}(\RR^2) = \dim\Sym_{216}(\RR^2) = 217. 
\]
To prove $\bfd$ is a MFA, 
it remains to show the maximal proper subarchitectures of $\bfd$, 
\[\bfs_1=(2,5,5,38,1),\qquad \bfs_2=(2,6,4,38,1),\qquad \bfs_3=(2,6,5,37,1)
\]
are not filling. 

To ease notation, we let $\dimV_{\bfa}$ denote the dimension of $\calV_{\bfa,6}$
For $\bfs_1$, 
repeated application of the recursive dimension bound gives 
\[ \begin{aligned} \dimV_{\bfs_1} &\leq \dimV_{(2,5)}+\dimV_{(5,5)}+\dimV_{(5,38)}+\dimV_{(38,1)}-5-5-38
\end{aligned} 
\]
Since
 $\dimV_{(m,n)}=mn$,
we have 
\begin{equation}\label{eq:s1-L4-counterexample}
\dimV_{\bfs_1} \leq 10+25+190+38-48 = 215 < 217.
\end{equation}

For $\bfs_2$, the recursive dimension bound gives 
$\dimV_{\bfs_2} \leq \dimV_{(2,6)} + \dimV_{(6,4)} + \dimV_{(4,38,1)} -6-4.$
Here $\dimV_{(2,6)}=12$, $\dimV_{(6,4)}=24$, 
while the ambient-dimension bound gives
\[ \dimV_{(4,38,1)} \leq \dim\Sym_6(\RR^4) = \binom{4+6-1}{6} = 84. \] 
Therefore, 
\begin{equation}\label{eq:s2-L4-counterexample}
\dimV_{\bfs_2} 
    \leq 12+24+84-10 = 110 
    < 217.
\end{equation}

Finally, for $\bfs_3$, the recursive and ambient dimension bounds give 
\begin{equation}\label{eq:s3-L4-counterexample}
\begin{aligned} 
\dimV_{\bfs_3} 
&\leq \dimV_{(2,6,5)}+\dimV_{(5,37)}+\dimV_{(37,1)}-5-37\\ &\leq 35+185+37-42\\ 
&=215<217. \end{aligned} 
 \end{equation}

Thus $\bfs_i$ is not filling for $i=1,2,3$ by \eqref{eq:s1-L4-counterexample}, \eqref{eq:s2-L4-counterexample}, and \eqref{eq:s3-L4-counterexample}, respectively.
Hence $\bfd$ is a nonunimodal minimal filling architecture, and the depth bound in \Cref{theorem:uni} is~sharp. 
\end{example}    

\section{Further Examples of MFAs}

We begin by classifying the MFAs of depth at most $7$ for $d_0=2$, $d_\depth=1$, and $r=2$. We then give \Cref{example:r=3,ex:counter-2-2} with other output widths and activation degrees and conclude with \Cref{ex:valley-mult}  exhibiting large defect and  multiple valleys.

\subsection{Classification through Depth $7$} 
Throughout this subsection, we fix 
$d_0=2$, $d_\depth=1$,$r=2$, 
and write 
\[ \dimV_{\bfa}:=\dim \calV_{\bfa,2}. \] 
The MFAs of depth at most $7$ are listed in \Cref{exhaustivelist}. 

\begin{table}[t!]
  \caption{
  MFAs
  for $d_0=2$, $d_\depth=1$, $r=2$, and $\depth\leq 7$.}
  \vspace{-.1in}
  \label{exhaustivelist}
  \begin{center}
    \begin{small}
      \begin{sc}
        \begin{tabular}{lcccr}
          \toprule
          $\depth$  & 
          Minimal Filling Architecture
          \\
          \midrule
          $\depth=2$ & $(2,2,1)$ \\
          $\depth=3$ & $(2,2,2,1)$\\
          $\depth=4$ & $(2,3,3,2,1)$\\
          $\depth=5$ & $(2,3,3,3,2,1)$\\
          $\depth=6$ & $(2,3,3,4,4,2,1)$\\
          & $(2, 3, 5, 5, 5, 4, 3, 1),(2, 3, 4, 5, 6, 4, 2, 1)$%
          ,%\\& 
          $(2, 3, 3, 4, 5, 6, 3, 1), (2, 3, 4, 5, 5, 5, 2, 1)$\\
          $\depth=7$
          & $(2, 3, 3, 5, 6, 4, 4, 1),(2, 3, 3, 5, 7, 4, 2, 1)$%
          ,%\\& 
          $(2, 3, 3, 4, 6, 5, 2, 1),(2, 3, 4, 5, 5, 4, 4, 1)$\\
          & $(2, 3, 4, 4, 5, 5, 4, 1),(2, 3, 3, 6, 6, 4, 3, 1)$%
          ,%\\& 
          $(2, 3, 3, 5, 5, 5, 4, 1),(2, 3, 4, 6, 5, 4, 3, 1)$\\
          & $(2, 3, 4, 5, 4, 6, 4, 1)$\\
          \bottomrule
        \end{tabular}
      \end{sc}
    \end{small}
  \end{center}
  \vskip -0.1in
\end{table}

\begin{lemma}\label{lemma:complete-through-depth-six} For $d_0=2$, $d_6=1$, and $r=2$, the unique MFA of depth $6$ is $(2,3,3,4,4,2,1)$. \end{lemma}

\begin{proof}
Let
\begin{equation}\label{eq:d-Lequals6}
\bfd=(2,d_1,d_2,d_3,d_4,d_5,1)
\end{equation}
be a filling architecture. We first claim that all five inequalities
\[
d_1\geq3,\qquad
d_2\geq3,\qquad
d_3\geq4,\qquad
d_4\geq4,\qquad
d_5\geq2
\]
must hold. Equivalently, if any one of these inequalities fails, then
$\bfd$ is not filling.

We illustrate the argument by proving that $d_1\geq3$.
Suppose first that $d_1\leq2$. By \Cref{lem:MFA-closed-under-increased-hidden-widths}, it suffices to consider $d_1=2$. The recursive dimension bound and the ambient dimension bound give 
\[ 
\begin{aligned} \dimV_{\bfd} &\leq \dimV_{(2,2)} + \dimV_{(2,d_2,d_3,d_4,d_5,1)} -2\\ &\leq 4+\dim\Sym_{16}(\RR^2)-2
&= 4+17-2<33. 
\end{aligned} 
\] 
Thus a filling architecture must satisfy $d_1\geq3$.    
The other four inequalities follow from analogous
applications of the recursive dimension bound (\Cref{Lemma:Recursive}), and provide the following statements:

\begingroup
\setlength{\abovedisplayskip}{0pt} \setlength{\belowdisplayskip}{0pt} \setlength{\abovedisplayshortskip}{0pt} \setlength{\belowdisplayshortskip}{0pt}

\begin{itemize}[nosep]
    \item If $\bfd$ is of the form in \eqref{eq:d-Lequals6} and $d_2=2$, then
\[ 
\begin{aligned} 
\dimV_{\bfd} &\leq \dimV_{(2,d_1,2)} + \dimV_{(2,d_3,d_4,d_5,1)} -2\\ &\leq 2\dim\Sym_2(\RR^2) + \dim\Sym_8(\RR^2) -2
&= 6+9-2<33. 
\end{aligned} 
    \]

    \item If $\bfd$ is of the form in \eqref{eq:d-Lequals6} and $d_3=3$, then
\[ 
\begin{aligned} 
\dimV_{\bfd} &\leq \dimV_{(2,d_1,d_2,3)} + \dimV_{(3,d_4,d_5,1)} -3\\ &\leq 3\dim\Sym_4(\RR^2) + \dim\Sym_4(\RR^3) -3
& = 15+15-3<33. 
\end{aligned} 
\]

    \item If $\bfd$ is of the form in \eqref{eq:d-Lequals6} and $d_4=3$, then
    \[ 
    \begin{aligned} \dimV_{\bfd} &\leq \dimV_{(2,d_1,d_2,d_3,3)} + \dimV_{(3,d_5,1)} -3\\ &\leq 3\dim\Sym_8(\RR^2) + \dim\Sym_2(\RR^3) -3
    &= 27+6-3<33. 
    \end{aligned} 
    \]

\item     
 If $\bfd$ is of the form in \eqref{eq:d-Lequals6} and $d_5=1$, then
  \[ 
  \begin{aligned} \dimV_{\bfd} &\leq \dimV_{(2,d_1,d_2,d_3,d_4,1)} +\dimV_{(1,1)} -1\\ &\leq \dim\Sym_{16}(\RR^2)+1-1 &= 17<33. \end{aligned} 
  \]

\end{itemize}
\endgroup

\medskip
\noindent Therefore, every filling architecture of depth $6$ satisfies
\[
(2,3,3,4,4,2,1)\leq\bfd.
\]
Applying \Cref{alg:randomized-image-dimension} to
$(2,3,3,4,4,2,1)$
returns the full ambient rank $33$, so this architecture is filling.
It is therefore contained in every filling architecture of depth $6$
and is itself filling. Hence it is the unique minimal filling
architecture of depth $6$.
\end{proof}

For depths $\depth\leq5$, analogous and shorter arguments give the MFAs listed in \Cref{exhaustivelist}. The completeness of the depth $7$ portion of the table is proved in \Cref{appendix:sec:verify}. There are exactly $13$ MFAs of depth $7$, exactly one of which is nonunimodal: $(2,3,4,5,4,6,4,1)$. This is the architecture studied in \Cref{s:counterexample}.

\subsection{Defect and expected dimension}

Defectivity (for fixed activation degree $r$) is a measure of an architecture's loss in expressivity which cannot be attributed to expected symmetries. Currently, there is no systematic characterization of defective architectures, although several works establish nondefectivity under conditions involving the activation degree \cite{FRWY,Massarenti2025,UBDC2025}. This motivates the next definition.

\begin{definition}
    If $\bfd=(d_0,\dots,d_\depth)$, then the \demph{expected dimension} of the neurovariety $\calV_{\bfd,r}$, denoted $\expdim\left(\calV_{\bfd,r}\right)$, equals the minimum 
    of the dimension of the ambient space and
    $\leftExpDimension$. The \demph{defect} of the neurovariety is defined to be the difference $\expdim \left(\calV_{\bfd,r}\right)-\dim\left(\calV_{\bfd,r}\right)$.
\end{definition}

The examples below were found using \Cref{alg:frontier-search}. 
For each example, fillingness is certified by a finite field Jacobian computation returning full ambient rank, 
while minimality is established by proving that every maximal proper subarchitecture is non-filling. 
In addition, we report the finite field Jacobian ranks obtained for these subarchitectures as their computed generic dimensions. Each example is therefore an additional counterexample to \Cref{conj:ktb}.

\begin{figure*}[t]
  \centering
  % \documentclass[tikz,border=5pt]{standalone}
% \usepackage{tikz}

% --- Layer color chooser (plain pdflatex compatible) ---
% Define 10 distinct colors; if you have more than 10 layers, we will cycle them.
\newcommand{\LayerColor}[1]{%
  \ifcase#1\relax
  \or red!50%
  \or orange!50%!black%
  \or magenta!50%
  %\or yellow!50%!black%
  \or brown!50%!black%
  %\or green!50%!black%
  \or violet!50%
  \or blue!50%
  \or green!50%!black%
    %\or green!50%!black%
%\or teal!50%!black%
  \or cyan!50%!black%
    \or teal!50%!black%
  \or cyan!50%!black%
  \else gray!70%
  \fi
}

%\begin{document}

\begin{tikzpicture}[
  neuron/.style={circle, draw=black!80, minimum size=2.6mm, 
  inner sep=2.0pt},
  conn/.style={draw=gray!55, line width=0.33pt, opacity=0.33}
]

% ----- spacing -----
\def\xsep{1.65}   % horizontal separation between layers
\def\ysep{0.37}  % vertical separation within a layer

% ----- architecture (EDIT THIS FREELY) -----
% Examples:
\def\LayerSizes{2,3,4,4,10,17,11,12,4,2}
%\def\LayerSizes{2,3,4,4,10,17,11,12}
%\def\LayerSizes{2,3,4,4,10,17,11,12}

% number of colors in the palette above
\def\NumColors{10}

% ----- PASS 1: draw nodes and store sizes LS1, LS2, ... -----
\foreach \n [count=\L] in \LayerSizes {%
  % Store size in a macro named \LS<layer>
  \expandafter\xdef\csname LS\L\endcsname{\n}%

  % Remember how many layers we have (final L after loop)
  \xdef\NumLayers{\L}%

  % Compute x coordinate for this layer (avoid inline arithmetic in coordinates)
  \pgfmathsetmacro{\x}{(\L-1)*\xsep}%

  % Center this layer vertically
  \pgfmathsetmacro{\offset}{(\n-1)/2}%

  % Choose a color index that cycles through 1..NumColors
  \pgfmathtruncatemacro{\cidx}{mod(\L-1,\NumColors)+1}%
  \edef\thiscolor{\LayerColor{\cidx}}%

  % Draw neurons
  \foreach \i in {1,...,\n} {%
    \pgfmathsetmacro{\y}{(\i-1-\offset)*\ysep}%
    \node[neuron, fill=\thiscolor] (N\L\i) at (\x,\y) {};%
  }%

  % % Label layer size above it (colored to match)
  % \pgfmathsetmacro{\ylabel}{(\offset+1.6)*\ysep}%
  % \node[font=\small, text=\thiscolor] at (\x,\ylabel) {$\n$};%
}

% ----- PASS 2: connect consecutive layers (only if >= 2 layers) -----
\pgfmathtruncatemacro{\LastLayer}{\NumLayers-1}%

\ifnum\NumLayers>1
  \foreach \L in {1,...,\LastLayer} {%
    \pgfmathtruncatemacro{\Lp}{\L+1}%
    \edef\a{\csname LS\L\endcsname}%
    \edef\b{\csname LS\Lp\endcsname}%

    \foreach \i in {1,...,\a} {%
      \foreach \j in {1,...,\b} {%
        \draw[conn] (N\L\i) -- (N\Lp\j);%
      }%
    }%
  }%
\fi

\end{tikzpicture}

% \end{document}
  \caption{A counterexample to the unimodal MFA conjecture with $d_\depth=2$ ($r=2$)
    \label{fig:dL-equals-two-counter-example}}
\end{figure*}

\begin{example}[A counterexample with $d_\depth=2$]\label{ex:counter-2-2}
Fix
$d_0=2$, $d_{\depth}=2$, $r=2$, and $\depth=9$. The PNN with architecture $\bfd=(2, 3, 4, 4, 10, 17, 11, 12, 4, 2)$ is minimal filling, {and is shown in \Cref{fig:dL-equals-two-counter-example}.}
For $i=1,\dots,\depth-1$, let $\bfs_i$ denote the architecture obtained by decreasing the $i$th hidden layer of $\bfd$ by one. 
Then 
\[
\left(\codim(\calV_{\bfs_i,r}
)\right)_{i=1}^{\depth-1}=(254,17,176,5,17,34,11,124).
\]
Moreover, $\expdim(\calV_{\bfs_i,r})=\dim(\ambientSpace)$ holds for each $i$, hence
$\defect(\calV_{\bfs_i,r})=\codim(\calV_{\bfs_i,r})$.
\end{example}

\begin{example}[A counterexample with $r=3$]\label{example:r=3}
We  found a counterexample to \Cref{conj:ktb} where  $r=3$.
The architecture $\bfd=(2, 3, 7, 6, 15, 8, 1)$ is nonunimodal and minimal filling. 
The ambient dimension is $244$, {and} the reported {neurovariety} dimensions
and defects are 
\[
\left(\dim(\calV_{\bfs_i,r})\right)_{i=1}^{\depth-1}=(84,240,223,233,231)
\quad \&\quad
\left(\operatorname{defect}(\calV_{\bfs_i,r})\right)_{i=1}^{\depth-1}=(156,0,4,2,2).
\]
Unlike \Cref{ex:counter-2-2}, not all of the non-filling architectures $\bfs_i$ are defective: for some $i$, the expected dimension is already strictly smaller than the ambient dimension.
This example shows that maximal subarchitectures of MFAs need not have large defect or have expected dimension equal to that of the ambient space.
\end{example}

\subsection{Examples with Multiple Valleys}\label{subsec:valley-certificate}

The main counterexample in \Cref{s:counterexample} has a single decrease followed by a later increase. To describe the more complicated hidden width profiles produced by our search, we introduce the notion of a valley. 

\begin{definition}\label{defn:valley} 
Let $\bfd=(d_0,\ldots,d_\depth)$
be an architecture. 
A \demph{valley} of $\bfd$ 
is a set of consecutive hidden layer
indices, possibly consisting of a single index,
\[ 
J=\{a,a+1,\ldots,b\} \subseteq \{1,\ldots,\depth-1\} 
\] 
such that 
\[ 
d_{a-1}>d_a=d_{a+1}=\cdots=d_b<d_{b+1}.
\] 
We say that $\bfd$ has $s$ valleys if it has exactly $s$ such sets. 
\end{definition}

\begin{figure*}[t]
    \centering
    % \documentclass[tikz,border=5pt]{standalone}
% \usepackage{tikz}

% --- Layer color chooser (plain pdflatex compatible) ---
% Define 10 distinct colors; if you have more than 10 layers, we will cycle them.
\newcommand{\LayerColor}[1]{%
  \ifcase#1\relax
  \or red!50%
  \or orange!50%!black%
  \or magenta!50%
  %\or yellow!50%!black%
  \or brown!50%!black%
  \or green!50%!black%
  \or violet!50%
  \or green!50%!black%
  \or blue!50%
    \or green!50%!black%
%\or teal!50%!black%
  \or cyan!50%!black%
  \else gray!70%
  \fi
}

%\begin{document}

\begin{tikzpicture}[
  neuron/.style={circle, draw=black!80, minimum size=2mm, 
  inner sep=2.0pt},
  conn/.style={draw=gray!55, line width=0.3pt, opacity=0.33}
]

% ----- spacing -----
\def\xsep{1.3}   % horizontal separation between layers
\def\ysep{0.37}  % vertical separation within a layer

% ----- architecture (EDIT THIS FREELY) -----
% Examples:
\def\LayerSizes{2,3,4,5,4,7,6,15,12,28,23,1}
%\def\LayerSizes{2,3,4,4,10,17,11,12}
%\def\LayerSizes{2,3,4,4,10,17,11,12}

% number of colors in the palette above
\def\NumColors{10}

% ----- PASS 1: draw nodes and store sizes LS1, LS2, ... -----
\foreach \n [count=\L] in \LayerSizes {%
  % Store size in a macro named \LS<layer>
  \expandafter\xdef\csname LS\L\endcsname{\n}%

  % Remember how many layers we have (final L after loop)
  \xdef\NumLayers{\L}%

  % Compute x coordinate for this layer (avoid inline arithmetic in coordinates)
  \pgfmathsetmacro{\x}{(\L-1)*\xsep}%

  % Center this layer vertically
  \pgfmathsetmacro{\offset}{(\n-1)/2}%

  % Choose a color index that cycles through 1..NumColors
  \pgfmathtruncatemacro{\cidx}{mod(\L-1,\NumColors)+1}%
  \edef\thiscolor{\LayerColor{\cidx}}%

  % Draw neurons
  \foreach \i in {1,...,\n} {%
    \pgfmathsetmacro{\y}{(\i-1-\offset)*\ysep}%
    \node[neuron, fill=\thiscolor] (N\L\i) at (\x,\y) {};%
  }%

  % % Label layer size above it (colored to match)
  % \pgfmathsetmacro{\ylabel}{(\offset+1.6)*\ysep}%
  % \node[font=\small, text=\thiscolor] at (\x,\ylabel) {$\n$};%
}

% ----- PASS 2: connect consecutive layers (only if >= 2 layers) -----
\pgfmathtruncatemacro{\LastLayer}{\NumLayers-1}%

\ifnum\NumLayers>1
  \foreach \L in {1,...,\LastLayer} {%
    \pgfmathtruncatemacro{\Lp}{\L+1}%
    \edef\a{\csname LS\L\endcsname}%
    \edef\b{\csname LS\Lp\endcsname}%

    \foreach \i in {1,...,\a} {%
      \foreach \j in {1,...,\b} {%
        \draw[conn] (N\L\i) -- (N\Lp\j);%
      }%
    }%
  }%
\fi

\end{tikzpicture}

% \end{document}
    \caption{A filling architecture with $\depth = 11$, $r = 2$, and three valleys. 
    The  Jacobian rank computations using \Cref{alg:randomized-image-dimension} provide evidence that the architecture is minimal~filling. 
    \label{fig:three-valleys}
    }
\end{figure*}

\begin{example}\label{ex:valley-mult}
Using  \Cref{alg:randomized-image-dimension},
we provide computational evidence of MFAs with $s$~valleys for $2\leq s\leq 5$. 
For all of these examples, we take the activation degree to be $r=2$. 
    \begin{itemize}[nosep]
        \item $2$-valley example: $\bfd=(2,3,4,7,9,8,13,7,16,10,1)$,
        \item $3$-valley example: $\bfd=(2,3,4,5,4,7,6,15,12,28,23,1)$,
        \item $4$-valley example: $\bfd=(2,3,4,5,4,7,6,14,9,36,29,67,11,1)$ 
        \item $5$-valley example: $\bfd=(2,3,4,6,5,11,9,35,25,100,38,143,13,50,40,1)$
    \end{itemize}
The computations certify these architectures are filling. 
The minimality conclusion is conditional on the sampled parameter points attaining the generic Jacobian ranks.
\end{example}

These computations suggest that minimal filling architectures may have increasingly complicated width profiles as the depth grows. We leave a systematic study of MFA with several valleys for future work.

\section{Conclusion}

Our results concern a fully connected feedforward PNN model with power activation. An important generalization is to allow broader polynomial activation families, for example activations expressed in orthogonal polynomial bases such as Chebyshev polynomials. A complementary direction is to investigate the polynomial networks ($\Pi$-nets) in \cite{CMBDPZ2022} using tools from neuroalgebraic geometry, to obtain a geometric characterization of their expressivity.

Another direction is to compare training dynamics on MFAs versus maximal non-filling architectures. Several works have already investigated optimization landscapes and function-space geometry for convolutional neural networks \cite{KMMT-2022-neu-ag} and  PNNs \cite{ABKPT-2026-neu-ag} through a neuroalgebraic geometry lens. But a direct comparison between MFA and maximal non-filling architectures is yet to be studied.

A new direction is to systematically characterize the boundary between non-filling and filling architectures. This means systematically understanding minimal filling architectures. An immediate first step is to find {more} MFAs when $r>2$. 

It is well known that polynomial activations present a trade-off: low-degree polynomials can limit expressivity, whereas higher-degree polynomials often lead to numerical instability \cite{Zhou2019} and gradient explosion \cite{Goyal2020}. Efforts to address these issues are already present in the literature \cite{Hossain2025}. One natural response is to pair low-degree polynomial activations with architectures that are \textit{a priori} highly expressive; in this context, understanding MFAs for power activations may inform architecture selection.
  
Lastly, we note that our approach also allows one to quickly find neurovarieties $\calV_{\bfd,r}\subseteq \Sym_{r^{\depth-1}}(\mathbb{R}^{d_0})^{\oplus d_\depth}$ of codimension $1$. After finding a minimal filling architecture, one can then focus on checking the maximal subarchitectures for potential hypersurfaces. An example we found is $(4,6,4,1)$ with $r=2$. This example coincides with one of the examples studied in \cite{Blekherman_2012} where the degree of the associated complex projective variety is shown to have degree $38,475$. Other small examples we found are $(2,3,4,5)$ with $r=2$, $(2,3,4,2)$ with $r=3$, and $(3,6,7,1)$ with $r=3$. It would be interesting to understand when $\calV_{\bfd,r}$ is of codimension $1$.

\subsection*{Acknowledgments}

We thank Kathl\'en Kohn and Joe Kileel for their comments on this project. We also thank the anonymous reviewers for the feedback that led to improvements in this article.

KD acknowledges that the material is based upon work supported by the National Science Foundation Graduate
Research Fellowship Program under Grant No. 2137424. Any opinions, findings, and conclusions or recommendations expressed in this material are those of the authors and do not necessarily reflect the views of the National Science Foundation.
KD also acknowledges further travel support from National Science Foundation Grant No. 2510307 during which some of this work was performed.
JIR's research is partially supported by the Alfred P. Sloan Foundation and by the National Science Foundation Grant No. 2510307.

\bibliography{bib-unimodal}

@book {CLO-2025,
    AUTHOR = {Cox, David A. and Little, John and O'Shea, Donal},
     TITLE = {Ideals, varieties, and algorithms---an introduction to
              computational algebraic geometry and commutative algebra},
    SERIES = {Undergraduate Texts in Mathematics},
   EDITION = {Fifth},
 PUBLISHER = {Springer, Cham},
      YEAR = {2025},
     PAGES = {xvi+656},
      ISBN = {978-3-031-91840-7; 978-3-031-91841-4},
   MRCLASS = {13P10 (13-01 14-01 14Qxx 68W30)},
  MRNUMBER = {4952933},
       DOI = {10.1007/978-3-031-91841-4},
       URL = {https://doi.org/10.1007/978-3-031-91841-4},
}

@book {invitation-nonlinear-algebra,
    AUTHOR = {Micha{\l}ek, Mateusz and Sturmfels, Bernd},
     TITLE = {Invitation to nonlinear algebra},
    SERIES = {Graduate Studies in Mathematics},
    VOLUME = {211},
 PUBLISHER = {American Mathematical Society, Providence, RI},
      YEAR = {2021},
     PAGES = {xiii+226},
      ISBN = {978-1-4704-5367-1},
   MRCLASS = {14-02 (13-02 13A50 13P10 14Q15 14T90 15A69 20G05)},
  MRNUMBER = {4423369},
MRREVIEWER = {Anton\ Dochtermann},
       DOI = {10.1090/gsm/211},
       url = {https://doi.org/10.1090/gsm/211},
}

@book {metric-ag,
    AUTHOR = {Breiding, Paul and Kohn, Kathl\'en and Sturmfels, Bernd},
     TITLE = {Metric algebraic geometry},
    SERIES = {Oberwolfach Seminars},
    VOLUME = {53},
 PUBLISHER = {Birkh\"auser/Springer, Cham},
      YEAR = {2024},
     PAGES = {xiv+215},
      ISBN = {978-3-031-51461-6; 978-3-031-51462-3},
   MRCLASS = {14Qxx (13Pxx 53Z50 62R01 90C23)},
  MRNUMBER = {4738534},
MRREVIEWER = {Scott McCallum},
       DOI = {10.1007/978-3-031-51462-3},
       url = {https://doi.org/10.1007/978-3-031-51462-3},
}

@inproceedings{ICLR2025_259e59fe-neu-ag,
 author = {Henry, Nathan and Marchetti, Giovanni Luca and Kohn, Kathl\'{e}n},
 booktitle = {International Conference on Learning Representations},
 editor1 = {Y. Yue and A. Garg and N. Peng and F. Sha and R. Yu},
 pages1 = {14400--14416},
 title = {Geometry of Lightning Self-Attention: Identifiability and Dimension},
 url = {https://proceedings.iclr.cc/paper_files/paper/2025/file/259e59fe23ebd09252647fed42949182-Paper-Conference.pdf},
 volume1 = {2025},
 year = {2025}
}

@article{
brandenburg2024-tropical-neu-ag,
title={The Real Tropical Geometry of Neural Networks for Binary Classification},
author={Marie-Charlotte Brandenburg and Georg Loho and Guido Montufar},
journal={Transactions on Machine Learning Research},
issn={2835-8856},
year={2024},
url={https://openreview.net/forum?id=I7JWf8XA2w},
note={}
}

@InProceedings{
pmlr-v80-zhang18i-tropical-neu-ag,
  title = 	 {Tropical Geometry of Deep Neural Networks},
  author =       {Zhang, Liwen and Naitzat, Gregory and Lim, Lek-Heng},
  booktitle = 	 {Proceedings of the 35th International Conference on Machine Learning},
  pages = 	 {5824--5832},
  year = 	 {2018},
  editor = 	 {Dy, Jennifer and Krause, Andreas},
  volume = 	 {80},
  series = 	 {Proceedings of Machine Learning Research},
  month = 	 {10--15 Jul},
  publisher =    {PMLR},
  url = 	 {https://proceedings.mlr.press/v80/zhang18i.html}
}

@article {ABKPT-2026-neu-ag,
    AUTHOR = {Arjevani, Yossi and Bruna, Joan and Kileel, Joe and Polak,
              Elzbieta and Trager, Matthew},
     TITLE = {Geometry and {O}ptimization of {S}hallow {P}olynomial
              {N}etworks},
   JOURNAL = {SIAM J. Appl. Algebra Geom.},
  FJOURNAL = {SIAM Journal on Applied Algebra and Geometry},
    VOLUME = {10},
      YEAR = {2026},
    NUMBER = {2},
     PAGES = {174--209},
      ISSN = {2470-6566},
   MRCLASS = {14Q15 (14N07 65 68T07)},
  MRNUMBER = {5061933},
       DOI = {10.1137/25M1732994},
       url = {https://doi-org.ezproxy.library.wisc.edu/10.1137/25M1732994},
}

@article {KMMT-2022-neu-ag,
    AUTHOR = {Kohn, Kathl\'en and Merkh, Thomas and Mont\'ufar, Guido and
              Trager, Matthew},
     TITLE = {Geometry of linear convolutional networks},
   JOURNAL = {SIAM J. Appl. Algebra Geom.},
  FJOURNAL = {SIAM Journal on Applied Algebra and Geometry},
    VOLUME = {6},
      YEAR = {2022},
    NUMBER = {3},
     PAGES = {368--406},
      ISSN = {2470-6566},
   MRCLASS = {68T07 (14J70 14P10 62R01 90C23)},
  MRNUMBER = {4459526},
       DOI = {10.1137/21M1441183},
       url = {https://doi-org.ezproxy.library.wisc.edu/10.1137/21M1441183},
}

@article{CMBDPZ2022, 
author = {Chrysos, Grigorios G. and Moschoglou, Stylianos and Bouritsas, Giorgos and Deng, Jiankang and Panagakis, Yannis and Zafeiriou, Stefanos}, 
title = {Deep Polynomial Neural Networks}, 
year = {2022}, 
issue_date1 = {Aug. 2022}, 
publisher = {IEEE Computer Society}, address = {USA}, 
volume = {44}, number = {8}, 
issn1 = {0162-8828}, 
url = {https://doi.org/10.1109/TPAMI.2021.3058891}, 
DOI = {10.1109/TPAMI.2021.3058891}, 
journal = {IEEE Trans. Pattern Anal. Mach. Intell.}, 
month1 = aug, 
pages = {4021–4034}, numpages = {14} }

@inproceedings{UBDC2025,
 author = {Usevich, Konstantin and Borsoi, Ricardo and D\'{e}rand, Clara and Clausel, Marianne},
 booktitle = {Advances in Neural Information Processing Systems},
 editor2 = {D. Belgrave and C. Zhang and H. Lin and R. Pascanu and P. Koniusz and M. Ghassemi and N. Chen},
 pages = {81809--81858},
 publisher = {Curran Associates, Inc.},
 title = {Identifiability of Deep Polynomial Neural Networks},
 url ={https://openreview.net/forum?id=MrUsZfQ9pC},
 url2 = {https://proceedings.neurips.cc/paper_files/paper/2025/file/7575b05ab751712f950029fb96f2f999-Paper-Conference.pdf},
 volume = {38},
 year = {2025}
}

@inproceedings{KTB,
    author    = {Kileel, J. and Trager, M. and Bruna, J.},
    title     = {{O}n the {E}xpressive {P}ower of {D}eep {P}olynomial {N}eural {N}etworks},
    booktitle = {Advances in Neural Information Processing Systems},
    volume1    = {32},
    year      = {2019},
  url       = {https://proceedings.neurips.cc/paper_files/paper/2019/hash/a0dc078ca0d99b5ebb465a9f1cad54ba-Abstract.html}    
}

@article{KLW,
    author  = {Kubjas, K. and Li, J. and Wiesmann, M.},
    title   = {{G}eometry of {P}olynomial {N}eural {N}etworks},
    journal = {Algebraic Statistics},
    volume  = {15},
    number  = {2},
    pages   = {295--328},
    year    = {2024},
    DOI     = {10.2140/astat.2024.15.295}
}

@article{FRWY,
    author  = {Finkel, Bella and Rodriguez, Jose Israel and Wu, Chenxi and Yahl, Thomas},
    title   = {{A}ctivation degree thresholds and expressiveness of polynomial neural networks},
    journal = {Algebraic Statistics},
    year    = {2025},
    volume  = {16},
    pages   = {113--130},
    DOI     = {10.2140/astat.2025.16.113}
}

@article{Massarenti2025,
    author={Massarenti, Alessio and Mella, Massimiliano},
    title={The {A}lexander-{H}irschowitz theorem for neurovarieties},
    journal = {arXiv:2511.19703},
    year    = {2025},
    DOI     = {10.48550/arXiv.2511.19703},
    url     = {https://arxiv.org/abs/2511.19703}
}

@inproceedings{Marchettietal2025,
    author    = {Marchetti, Giovanni Luca and Shahverdi, Vahid and Mereta, Stefano and Trager, Matthew and Kohn, Kathl\'{e}n},
    title     = {{P}osition: {A}lgebra {U}nveils {D}eep {L}earning -- {A}n {I}nvitation to {N}euroalgebraic {G}eometry},
    booktitle = {Proceedings of the 42nd International Conference on Machine Learning},
    year      = {2025},
    series2    = {Proceedings of Machine Learning Research},
    publisher = {PMLR},
    DOI       = {10.48550/arXiv.2501.18915}
}

@article{Zhou2019,
    author  = {Zhou, Jun and Qian, Huimin and Lu, Xinbiao and Duan, Zhaoxia and Huang, Haoqian and Shao, Zhen},
    title   = {Polynomial activation neural networks: {M}odeling, stability analysis and coverage {BP}-training},
    journal = {Neurocomputing},
    volume  = {359},
    pages   = {227--240},
    year    = {2019},
    issn    = {0925-2312},
    DOI     = {10.1016/j.neucom.2019.06.004}
}

@INPROCEEDINGS{Goyal2020,
  author={Goyal, Mohit and Goyal, Rajan and Lall, Brejesh},
  booktitle={2020 International Joint Conference on Neural Networks (IJCNN)}, 
  title={{I}mproved {P}olynomial {N}eural {N}etworks with {N}ormalised {A}ctivations}, 
  year={2020},
  volume={},
  number={},
  pages={1-8},
  DOI={10.1109/IJCNN48605.2020.9207535}}

@article{Hossain2025,
      title={A {T}raining {F}ramework for {O}ptimal and {S}table {T}raining of {P}olynomial {N}eural {N}etworks}, 
      author={Forsad Al Hossain and Tauhidur Rahman},
      journal = {arXiv:2505.11589},
      year={2025},
      eprint={2505.11589},
      archivePrefix={arXiv},
      primaryClass={cs.LG},
      url={https://arxiv.org/abs/2505.11589}, 
}

@article{AlexanderHirschowitz1995,
    AUTHOR = {Alexander, J. and Hirschowitz, A.},
     TITLE = {Polynomial interpolation in several variables},
   JOURNAL = {J. Algebraic Geom.},
  FJOURNAL = {Journal of Algebraic Geometry},
    VOLUME1 = {4},
      YEAR = {1995},
    NUMBER1 = {2},
     PAGES1 = {201--222},
      ISSN = {1056-3911,1534-7486},
   MRCLASS = {14N10 (14F17 14Q15)},
  MRNUMBER = {1311347},
MRREVIEWER = {Fyodor\ L.\ Zak},
URL = {https://www.ams.org/journals/jag/1995-04-02/S1056-3911-1995-1311347-7/},
}

@article{Abo2010,
title = {On non-defectivity of certain Segre-Veronese varieties},
journal = {Journal of Symbolic Computation},
volume1 = {45},
number1 = {12},
pages1 = {1254-1269},
year = {2010},
note1 = {MEGA’2009},
issn1 = {0747-7171},
DOI =  {10.1016/j.jsc.2010.06.008},
url = {https://www.sciencedirect.com/science/article/pii/S074771711000088X},
author = {{A}bo, {H}irotachi}
}

@article{AboBrambilla2009,
author = {Hirotachi Abo and Maria Chiara Brambilla},
title = {{Secant Varieties of Segre-Veronese Varieties $\mathbb{P}^m \times \mathbb{P}^n$ Embedded by $\mathcal{O}(1,2)$}},
volume1 = {18},
journal = {Experimental Mathematics},
number1 = {3},
publisher1 = {A K Peters, Ltd.},
pages1 = {369 -- 384},
keywords1 = {defectivity, secant varieties, Segre-Veronese varieties},
year = {2009},
DOI = {10.1080/10586458.2009.10129051},
}

@article{Blekherman_2012,
   title={Algebraic boundaries of Hilbert’s SOS cones},
   volume={148},
   ISSN={1570-5846},
   url={http://dx.doi.org/10.1112/S0010437X12000437},
   DOI={10.1112/s0010437x12000437},
   number={6},
   journal={Compositio Mathematica},
   publisher={Cambridge University Press (CUP)},
   author={Blekherman, Grigoriy and Hauenstein, Jonathan and Ottem, John Christian and Ranestad, Kristian and Sturmfels, Bernd},
   year={2012},
   month=Oct, pages={1717–1735} }
\bibliographystyle{jir-siam}

\appendix

\newpage
\section{Jacobians, Dimension, and Finite Fields}\label{appendix:A}

In this appendix, we provide some background for  \Cref{lemma:counterfill} and finite field computations.

\begin{definition}
    Let $F:\mathbb{R}^{n}\to \mathbb{R}^{m}$ be a polynomial map i.e. 
    \[F(x_1,\dots,x_n)=(f_1(x_1,\dots,x_n),\dots,f_m(x_1,\dots,x_n))
    \]
    where the $f_i$ are polynomials in $n$-variables. The \textbf{Jacobian} of $F$ is defined to be the $m\times n$-matrix
    \[
    \operatorname{Jac}F=\begin{pmatrix}
        \frac{\partial f_1}{\partial x_1} & \dots & \frac{\partial f_1}{\partial x_n}\\
        \vdots & \ddots & \vdots \\
        \frac{\partial f_m}{\partial x_1} & \dots & \frac{\partial f_m}{\partial x_n}
    \end{pmatrix}.
    \]
    This is an $m\times n$-matrix whose entries are polynomials in $n$-variables. The \textbf{rank of $\operatorname{Jac}F$ at a point $x\in \mathbb{R}^n$} is the rank of the $m\times n$-matrix with real entries $\operatorname{Jac}F(x)$ obtained by evaluating each polynomial entry of $\operatorname{Jac}F$ at the point $x\in \mathbb{R}^n$.
\end{definition}

In algebraic geometry, a property that holds outside a proper closed subvariety is said to hold \emph{generically}.
The key fact regarding $\operatorname{Jac}F(x)$ is that there exists a Zariski closed
subvariety $Z\subseteq \mathbb{R}^n$ such that the rank of $\operatorname{Jac}F(x)$ is constant for $x\in \mathbb{R}^n\setminus Z$.

\begin{theorem}\label{appendix:theorem1} 
If $F:\mathbb{R}^n\to \mathbb{R}^m$ is a polynomial map, then there exists a Zariski open subset $U\subseteq \mathbb{R}^n$ such that for all $x\in U$,
\[
\operatorname{rank}(\operatorname{Jac}F(x))=\dim \overline{\operatorname{Im}(F)}.
\]
where the bar denotes taking the Zariski closure inside $\RR^m$. Furthermore, outside of $U$, the rank can only drop. 
\end{theorem}

\begin{proof}
    See Lemma A.1 of \cite{UBDC2025}.
\end{proof}

The value of $\operatorname{rank}(\operatorname{Jac}F(x))$ on $U$ in the previous theorem is referred to as the \textbf{generic rank}. {The previous theorem implies, for generic $x\in\mathbb{R}^n$,
\[
\operatorname{rank}(\operatorname{Jac}F(x))
=
\dim \overline{\operatorname{Im}(F)}.
\]
Accordingly, if $\operatorname{Jac}\parameterMap{\bfd,r}(\theta)$ has full {row} rank for some $\theta$, then $\calV_{\bfd,r}$ is filling.
In practice, such a $\theta$ is found by a random choice of parameters.  
} {For more details, the reader {can} refer to Section 3.3 of \cite{KTB}.}

\begin{example}[Linear maps]\label{ex:linear-maps}
    Let $F:\mathbb{R}^n\to \mathbb{R}^m$ be a linear map defined by
    $F(x)=Ax$ for some $m\times n$
    matrix $A\in \mathbb{R}^{m\times n}$. 
    Then
    \[
    \operatorname{Jac}F(x)=A
    \]
    for every $x\in\mathbb{R}^n$, so the rank of the Jacobian is constant and equal
    to $\operatorname{rank}(A)$. Since the image of $F$ is the linear subspace spanned by the columns of $A$, we also have
    \[
    \dim \overline{\operatorname{Im}(F)}=\dim \operatorname{Im}(A)=\operatorname{rank}(A).
    \]
    Thus \Cref{appendix:theorem1} holds in this case.\end{example}

\begin{example}[A nonlinear map]
Consider the map $F:\mathbb{R}\to\mathbb{R}$, $F(x)=x^2$.
    Here
    \[
    \operatorname{Jac}F(x)=[2x].
    \]
    Hence
    \[
    \operatorname{rank}(\operatorname{Jac}F(x))=
    \begin{cases}
        1, & x\neq 0,\\
        0, & x=0.
    \end{cases}
    \]
    Thus the generic rank is $1$. On the other hand,
    $\operatorname{Im}(F)=[0,\infty)$,
    whose Zariski closure in $\mathbb{R}$ is all of $\mathbb{R}$. Therefore
    $\dim \overline{\operatorname{Im}(F)}=1$,
    in agreement with the \Cref{appendix:theorem1}.
\end{example}

\begin{example}[A quadratic map with one-dimensional image]
    Let $q(x_1,x_2,x_3)=x_1^2+x_2^2+x_3^2$. Consider the {polynomial}
    map $F:\RR^3\to \RR^2$ defined by 
    \[
    F(x_1,x_2,x_3)=(q(x_1,x_2,x_3),2q(x_1,x_2,x_3) ).
    \]
    Then
    \[
\operatorname{Jac}F(x_1,x_2,x_3)=
    \begin{pmatrix}
        2x_1 & 2x_2 & 2x_3\\
        4x_1 & 4x_2 & 4x_3
    \end{pmatrix},
    \]
    whose second row is always twice the first. Therefore the Jacobian has generic
    rank $1$. The image of $F$ lies in the line
    \[
    \{(u,2u):u\in\mathbb{R}\}\subseteq \mathbb{R}^2,
    \]
    and its Zariski closure is exactly this line, which has dimension $1$. 
\end{example}

\begin{corollary}\label{appendix:corollary-integer-coefficients}
    If $F:\mathbb{R}^n\to \mathbb{R}^m$ is a polynomial map with \textbf{integer coefficients},  
    {then there exists a Zariski dense open subset $U\subseteq\mathbb{R}^n$ such that for every
    $x\in \mathbb{Z}^n\cap U$, 
    one has
    \[
    \operatorname{rank}(\operatorname{Jac}F(x))
    =
    \operatorname{rank}(\operatorname{Jac}F(x)\bmod p)
    \]
    for all but finitely many primes $p$.}   
\end{corollary}

In \Cref{appendix:corollary-integer-coefficients}, $\operatorname{Jac}F(x)\bmod p$ is obtained by evaluating the entries of the matrix $\operatorname{Jac}F$ at $x$ and then reducing modulo $p$ to a $\mathbb{F}_p$-valued matrix. Since reduction modulo $p$ can only ever reduce the rank, if $x \in \mathbb{Z}^n$ and $p$ a prime, then $\operatorname{rank}(\operatorname{Jac}F(x)) \geq  \operatorname{rank}(\operatorname{Jac}F(x) \bmod p)$.

Unless otherwise specified, we apply \Cref{alg:randomized-image-dimension} with primes $\mathcal P=\{100003,100153\}$ and $T=2$.
For each prime $p$, the entries of the weight matrices are sampled as random elements of $\mathbb F_p$ using Sage's \texttt{random\_matrix} constructor. Thus the implementation uses two random parameter points over each prime field and returns the largest of the computed ranks.

\begin{algorithm}[H] 
\caption{Randomized finite field Jacobian rank search} \label{alg:randomized-image-dimension} 
\begin{algorithmic}[1] 
\REQUIRE A polynomial map $F:\RR^n\to\RR^m$ with integer coefficients; 
a finite set of primes $\mathcal P$; 
and a number of trials $T$ 
\STATE $\rho\gets0$ 
\FOR{$p\in\mathcal P$} \FOR{$t=1,\ldots,T$} 
\STATE Choose $x_{p,t}\in\mathbb F_p^n$ uniformly at random 
\STATE $\rho\gets\max\left\{
    \rho,\,
    \operatorname{rank}_{\mathbb F_p} \bigl(\operatorname{Jac}F(x_{p,t})\bigr)
    \right\}$ 
\ENDFOR \ENDFOR \STATE 
\textbf{return} $\rho$ \end{algorithmic}
\end{algorithm}

For neurovarieties, the parameter map has integer coefficients.  If \Cref{alg:randomized-image-dimension} returns the dimension of the ambient space for $\parameterMap{\bfd,r}$, then $\calV_{\bfd,r}$ is filling. More generally, if one has a rigorous upper bound $\dim\calV_{\bfd,r}\leq B$ and the algorithm returns $B$, then $\dim\calV_{\bfd,r}=B.$ Here and throughout, applying \Cref{alg:randomized-image-dimension} to an architecture $\bfd$ means applying it to the associated parameter map $\parameterMap{\bfd,r}$.

\begin{example}[Linear maps revisited]
Suppose $F(x)=Ax$ with $A\in \mathbb{Z}^{{m\times n}}$. Then
\[
\operatorname{Jac}F(x)=A
\]
for all $x\in\mathbb{Z}^n$. Let $r=\operatorname{rank}(A)$. The rank of $A\bmod p$ drops
below $r$ exactly when all $r\times r$ minors of $A$ vanish modulo $p$. Hence any such
prime $p$ must divide every nonzero $r\times r$ minor of $A$. Since there are only finitely
many minors, and each nonzero minor has only finitely many prime divisors, there are only
finitely many primes for which the rank drops. Thus
\[
\operatorname{rank}(A)=\operatorname{rank}(A\bmod p)
\]
for all but finitely many primes $p$.
\end{example}

\begin{example}
    For the architecture $\mathbf{d}=(2,1,1)$ and $r=2$, the parameter map is given by 
    \[
    \parameterMap{\mathbf{d},r}:
    {\RR^{1\times 2}}\times \RR^{1\times 1}\to \RR^3\qquad \left(\begin{bmatrix}
        w_{11} & w_{12}
    \end{bmatrix},\begin{bmatrix}
        w_{2}
    \end{bmatrix}\right)\mapsto 
    (w_{2}w_{11}^2,2w_2w_{11}w_{12},w_2w_{12}^2).
    \]
    {Its} Jacobian matrix is
    \[
    \operatorname{Jac}(\parameterMap{\mathbf{d},r})=\begin{bmatrix}
    2w_2w_{11} & 0 & w_{11}^2\\
    2w_2w_{12} & 2w_2w_{11} & 2w_{11}w_{12}\\
    0 & 2w_2w_{12} & w_{12}^2
    \end{bmatrix}.
    \]
    {The determinant of this matrix is identically zero, so 
    $\operatorname{Jac}(\parameterMap{\mathbf{d},r})(w_{11},w_{12},w_2)$ has rank at most $2$ for all $(w_{11},w_{12},w_2)\in \mathbb{R}^3$. 
    Moreover, at the point $(w_{11},w_{12},w_2)=(1,0,1)$, the Jacobian is
    \[
    \begin{bmatrix}
    2 & 0 & 1\\
    0 & 2 & 0\\
    0 & 0 & 0
    \end{bmatrix},
    \]
    which has rank $2$. Hence the generic rank is $2$, and therefore
$\dim \mathcal{V}_{\mathbf{d},r}=2.$

    The rank of the Jacobian at a point drops below $2$ exactly when all nine $2\times 2$ minors vanish.
    For the integer point $(w_{11},w_{12},w_2)=(1,0,1)$, the reduced Jacobian modulo $p$
    has rank $2$ for every prime $p\neq 2$. 
    Thus, a single computation modulo any odd
    prime certifies that the generic rank is $2$.
    }
\end{example}

Lastly, the task of efficiently computing $\operatorname{Jac}\Psi_{\bfd,r}$ is resolved by backpropagation. The interested reader can refer to Section 6.4 of \cite{KLW} or to Section 3.3 of \cite{KTB}.
More broadly, for textbook references on applications of algebraic geometry we refer to 
\cite{invitation-nonlinear-algebra,metric-ag}.

The computer computations in this article were carried out using a public implementation found at the GitHub repository
\begin{center}
\url{https://github.com/daokevin06/MFA_PNNs}
\end{center}
which contains the finite-field Jacobian rank computations, fillingness verifications, and frontier search code.

\section{Parameter choices for  \Cref{lemma:counterfill}}\label{appendix:parameterchoices}
The weight matrices for the computation in \Cref{lemma:counterfill} are provided below. We used the prime $p=101$.

\begingroup \small \setlength{\arraycolsep}{3pt} \renewcommand{\arraystretch}{0.9}

\begin{center}
\begin{tabular}{ccc}
$W_1 = \begin{bmatrix}
94&6\\
9&7\\
58&11    
\end{bmatrix}$

& 

$W_2=\begin{bmatrix}96&20&67\\
83&52&74\\
30&41 &2\\
3&90 &1
\end{bmatrix}$

&

$W_3=\begin{bmatrix}23&36&99&71\\
27&94&56&3\\
91&50&65&62\\
24&20&30&76\\
41&49&76&81
\end{bmatrix}$

\\\\

$W_4=\begin{bmatrix} 94 &55 &89 &90 &60\\
100 &41 &6 &49 &83\\
78 &56 &50 &3 &56\\
16 &46 &15 &10 &16
\end{bmatrix}$

&

$W_5=\begin{bmatrix}51&37&71&85\\
97&23&57&56\\
12&47&41&60\\
31&60&63&15\\
54 &2&57&60\\
40 &2&34&83
\end{bmatrix}$

& 

$W_6=\begin{bmatrix}57&90 &8&19&55&36\\
8&35&94&89&25&36\\
8&55 &6&69&14&28\\
59 &7&85&82&60 &4
\end{bmatrix}$

\\\\

$W_7=
\begin{bmatrix}
99&42&87&98
\end{bmatrix}$
\end{tabular}
\end{center}

\endgroup

\section{Search Algorithm}\label{sec:search-algorithm}

\subsection{Core algorithm}
Our algorithm for finding MFAs exploits \Cref{lem:MFA-closed-under-increased-hidden-widths} and the definition of minimality. The pseudocode is provided in \Cref{alg:frontier-search}.  The difficulty in finding MFAs is due to the size of the search space growing exponentially with the number of hidden layers. To efficiently  find MFAs $(d_0,\bfa,d_\depth)$, we employ a frontier search over $\bfa=(d_1,\dots,d_{\depth-1})\in [m,n]^{\depth-1}$. 
We first prune the search space using naive dimension bounds, e.g., $\dim\left(\calV_{\bfd,r}\right)\leq\leftExpDimension$. 
If there were any previous searches already performed, then we prune the search space using the results. Then, we repeatedly test $(d_0,\bfa,d_\depth)$ from among the remaining possibilities for filling architectures $(d_0,\bfa,d_\depth)$:

\begin{itemize}[nosep]
    \item If $(d_0,\bfa,d_\depth)$ is filling, we prune all $\bfa'> \bfa$.
    \item If $(d_0,\bfa,d_\depth)$ is not filling, we prune all $\bfa'< \bfa$.
\end{itemize}

Eventually this terminates to a finite set of candidates. We verify each candidate by computing the dimension of its subarchitectures. 
If verification suggests a candidate is not minimal, then we adjust our $m$ and $n$ and run the search again. Using previous search history to inform subsequent ones allowed us to quickly prune the search~space.

\begin{remark}
To determine if 
$(d_0,\bfa,d_\depth)$ is filling, our 
implementation of  \Cref{alg:frontier-search}, relies on finite field computations and generic choices of parameters. 
Accordingly, the output is correct for almost all primes and parameter choices. We then certify the result a posteriori by proving the relevant dimension bounds, as shown in \Cref{s:counterexample}.
\end{remark}

\setlength{\textfloatsep}{8pt}% Remove \textfloatsep
\begin{algorithm}[bht]
  \caption{Frontier search for hidden layers' widths
  }  
  \label{alg:frontier-search}
  \begin{algorithmic}[1]
    \STATE {\bfseries Input:} depth $\depth$, fixed endpoints $d_0,d_\depth$, widths range $[m,n]$, exponent $r$, integer $B$, seed $\mathsf{seed}$
    \STATE {\bfseries Initialize:} $\mathcal{F}_{\min}\leftarrow\emptyset$ \hfill (minimal filling antichain)
    \STATE {\bfseries Initialize:} $\mathcal{N}_{\max}\leftarrow\emptyset$ \hfill (maximal non-filling antichain)
    \STATE {\bfseries Initialize:} $\mathsf{rng}\leftarrow \mathsf{seed}$

    \FOR{$t=1$ {\bfseries to} $B$}
      \STATE Propose a candidate $\mathbf{a}\in\{m,\dots,n\}^{\depth-1}$ using $\mathsf{rng}$ 
      \IF{$\exists\,\mathbf{f}\in\mathcal{F}_{\min}\ \text{with}\ \mathbf{a}\geq \mathbf{f}$
      }
        \STATE {\bfseries continue} \hfill (already implied filling)
      \ENDIF
      \IF{$\exists\,\mathbf{q}\in\mathcal{N}_{\max}\ \text{with}\ \mathbf{a}\leq \mathbf{q}$}
        \STATE {\bfseries continue} \hfill (already implied non-filling)
      \ENDIF

      \IF{$(d_0,\mathbf{a},d_\depth)$ is filling}
        \STATE $\mathcal{F}_{\min}\leftarrow \mathcal{F}_{\min}\cup\{\mathbf{a}\}$ 
        \hfill (keep minimal) 
        \STATE Remove any $\mathbf{f}\in\mathcal{F}_{\min}$ with $\mathbf{f}> \mathbf{a}$ 
      \ELSE
        \STATE $\mathcal{N}_{\max}\leftarrow \mathcal{N}_{\max}\cup\{\mathbf{a}\}$ \hfill (keep maximal) 
        \STATE Remove any $\mathbf{q}\in\mathcal{N}_{\max}$ with $\mathbf{q}< \mathbf{a}$ 
      \ENDIF
    \ENDFOR

    \STATE {\bfseries Output:} $\mathcal{F}_{\min}$
  \end{algorithmic}
\end{algorithm}

\subsection{Termination of Algorithm}\label{subsection:TerminationOfAlgorithm}

In this subsection, we fix the input and output widths, activation degree, and depth. 
We derive explicit upper bounds on the hidden widths of an MFA. These bounds reduce the classification of MFAs to a finite search.

\begin{theorem}\label{theorem:algorithmtermination}
    If the architecture $\bfd=(d_0,d_1,\dots,d_\depth)$ is an MFA for activation degree~$r$, then its hidden layers satisfy
    \[
    d_i\leq \binom{d_0+r^i-1}{d_0-1}\text{ for }1\leq i\leq \depth-1.
    \]
\end{theorem}

\begin{proof}
Fix $i$, and set $N=\dim\Sym_{r^i}(\RR^{d_0})
 =\binom{d_0+r^i-1}{d_0-1}$.
Suppose $d_i>N$. For any choice of weights, write the activated
output of layer $i$ as $A\mathbf m(x)$, where $\mathbf m(x)$
is the vector of degree-$r^i$ monomials and
$A\in\RR^{d_i\times N}$ tracks the coefficients.

Since $\operatorname{rank}A\leq N$, select $N$ rows of $A$ that span its row space. 
Let $B\in \RR^{N\times N}$ be the matrix formed by these rows.
Then $A=CB$ for some $C\in\RR^{d_i\times N}$. Retain the corresponding neurons by keeping their rows of $W_i$, and replace $W_{i+1}$ by $W_{i+1}C$.
The next preactivation is unchanged by this replacement $(W_{i+1}C)B\mathbf m(x)=W_{i+1}A\mathbf m(x)$.
Thus every map realized by $\bfd$ is realized after reducing $d_i$ to $N$. 
The smaller architecture is filling, contradicting
minimality.
\end{proof}

\begin{remark} \Cref{theorem:algorithmtermination} reduces the classification of MFAs to a finite exhaustive search. 
Although such a search necessarily terminates, the candidate set grows rapidly with depth and quickly becomes computationally impractical. 
This motivates the study of better theoretical bounds and properties of minimal filling architectures that can further reduce the search space. 
\end{remark}

\section{Completeness of the MFA list for~$\depth\leq7$}
\label{appendix:sec:verify}

In this appendix, we prove that \Cref{exhaustivelist} is a complete list of MFAs for
$d_0=2$, $d_\depth=1$, and $\depth\leq 7$.

\begin{lemma}[Dickson's lemma]\label{lemma:dickson}
Let $\mathbb{N}$ denote the set of positive integers, and equip
$\mathbb{N}^r$ with the componentwise partial order:
\[
(a_1,\dots,a_r)\leq(b_1,\dots,b_r)
\quad\Longleftrightarrow\quad
 a_i\leq b_i\text{ for all }i.
\]
Then every subset $\mathcal{S}\subseteq\mathbb{N}^r$ has only finitely many minimal elements.
\end{lemma}

Dickson's lemma implies that the set of minimal filling architectures is finite for fixed $r$, $\depth$, $d_0$, and $d_\depth$. For each $\depth\leq 6$, \Cref{exhaustivelist} contains a unique MFA. The completeness of the $\depth=6$ portion was proved in \Cref{lemma:complete-through-depth-six}; the smaller depths follow by analogous and shorter arguments.
It remains to treat depth $\depth=7$.

\subsection{Four dimension computations}

The main additional difficulty in the $\depth=7$ case, compared with
the $\depth=6$ case, is that the completeness proof requires the four
neurovariety dimension computations in
\Cref{lem:L7-dimension-counts}.
We establish them using recursive
dimension bounds, the Alexander--Hirschowitz theorem, and finite-field
Jacobian-rank computations. For one of the four computations, we also
use a defectivity result for the Segre--Veronese variety associated
with $(3,5,6)$.

\begin{theorem}[Alexander--Hirschowitz {\cite{AlexanderHirschowitz1995}}]
\label{Theorem:AH}
Let $\bfd=(d_0,d_1,1)$. Then
\[
\dim\calV_{\bfd,r}
=
\min\left\{d_0d_1,\binom{d_0+r-1}{r}\right\},
\]
except in the following cases:
\begin{itemize}[nosep]
\item if $r=2$ and $2\leq d_1\leq d_0-1$, then
$\dim\calV_{\bfd,r}=d_0d_1-\binom{d_1}{2}$;
\item if $r=3$, $d_0=5$, and $d_1=7$, then $\dim\calV_{\bfd,r}=34$;
\item if $r=4$, $d_0=3$, and $d_1=5$, then $\dim\calV_{\bfd,r}=14$;
\item if $r=4$, $d_0=4$, and $d_1=9$, then $\dim\calV_{\bfd,r}=34$;
\item if $r=4$, $d_0=5$, and $d_1=14$, then $\dim\calV_{\bfd,r}=69$.
\end{itemize}
\end{theorem}

\begin{theorem}\label{Theorem:AboBrambilla}
Let $r=2$ and $\bfd=(d_0,d_1,d_2)$. If
\[
\binom{d_0-1+r}{r}-(d_0-1)
<d_1<
\min\left\{d_2,\binom{d_0-1+r}{r}\right\},
\]
then $\calV_{\bfd,r}$ is defective. In particular, $\calV_{(3,5,6),2}$ is defective.
\end{theorem}

\begin{proof}
See \cite[Remark~4.4]{AboBrambilla2009}.
\end{proof}

\begin{lemma}\label{lem:L7-dimension-counts}
If $r=2$, then
\[ \dimV_{(5,5,3,1)}=32,\qquad \dimV_{(3,5,6,4,3,1)}=58,\qquad \dimV_{(4,5,5,3,1)}=47,\qquad \dimV_{(5,4,3,1)}=25. \]
\end{lemma}

\begin{proof}
For each of the four architectures, \Cref{alg:randomized-image-dimension} returns the
displayed value as a finite field Jacobian rank lower bound. It remains to prove the
matching upper bound.

By \Cref{Theorem:AH},
$\dimV_{(5,3,1)}=12$.
The recursive dimension bound gives
\begin{equation}\label{eq:5531isleq32}
\dimV_{(5,5,3,1)}
\leq \dimV_{(5,5)}+\dimV_{(5,3,1)}-5
=25+12-5=32.
\end{equation}

The expected dimension of $\calV_{(3,5,6),2}$ is $36$.
By \Cref{Theorem:AboBrambilla}, this neurovariety is defective, and hence $\dimV_{(3,5,6)}\leq35$. Moreover, \Cref{Theorem:AH} gives $\dimV_{(4,3,1)}=4\cdot3-\binom{3}{2}=9$.
Therefore, repeated application of the recursive dimension bound yields
\begin{equation}\label{eq:356431isleq58}
\begin{aligned}
\dimV_{(3,5,6,4,3,1)}
&\leq \dimV_{(3,5,6)}+\dimV_{(6,4)}+\dimV_{(4,3,1)}-6-4
&\leq35+24+9-10=58.
\end{aligned}
\end{equation}
The expected dimension bound gives $\dimV_{(4,5,5)}\leq 40$.
Thus
\begin{equation}\label{eq:45531isleq47}
\dimV_{(4,5,5,3,1)}
\leq\dimV_{(4,5,5)}+\dimV_{(5,3,1)}-5
\leq40+12-5=47.
\end{equation}
Finally, we have
\begin{equation}\label{eq:5431isleq25}
\dimV_{(5,4,3,1)}
\leq\dimV_{(5,4)}+\dimV_{(4,3,1)}-4
=20+9-4=25.
\end{equation}
The respective upper bounds 
\eqref{eq:5531isleq32},
\eqref{eq:356431isleq58},
\eqref{eq:45531isleq47},
\eqref{eq:5431isleq25},
agree with the finite field Jacobian rank lower bounds from \Cref{alg:randomized-image-dimension}, 
proving all four
equalities.
\end{proof}

\begin{remark}
For the depth $7$ classification, the recursive and expected dimension bounds, together with the dimension computations above, are sufficient. 
An earlier version of the proof used additional results on defectivity and nondefectivity of secant varieties of Segre--Veronese varieties \cite{Abo2010,AboBrambilla2009} to obtain sharper bounds for some intermediate neurovarieties. 
These results are not needed here, but may be important for classifications in larger depth, where the recursive and expected dimension bounds may not be strong enough to prove that every remaining region is non-filling. 
\end{remark}

\subsection{Thirteen MFAs}
We now prove completeness for $\depth=7$. We first derive necessary lower bounds on the hidden widths of a filling architecture. After removing the upward closures of the thirteen listed filling architectures, the remaining possibilities form fifteen rectangular regions, each of which we show to be non-filling.

\begin{lemma}\label{lemma:complete-through-depth-seven}
For $d_0=2$, $d_7=1$, and $r=2$, the thirteen MFAs of depth
$\depth=7$ are exactly the architectures listed in
\Cref{exhaustivelist}.
\end{lemma}

\begin{proof}
Let
\begin{equation}\label{eq:d-Lequals7}
\bfd=(2,d_1,d_2,d_3,d_4,d_5,d_6,1)
\end{equation}
be a filling architecture.

\smallskip
\noindent\emph{Necessary width bounds.}
We first claim that
\[
d_1\geq3,\qquad
 d_2\geq3,\qquad
 d_3\geq4,\qquad
 d_4\geq4,\qquad
 d_5\geq4,\qquad
 d_6\geq2.
\]
Equivalently, if any one of these inequalities fails, then $\bfd$ is
not filling.

We illustrate the argument by proving that $d_1\geq3$. Suppose that $d_1\leq2$. 
By \Cref{lem:MFA-closed-under-increased-hidden-widths}, it suffices to consider the largest width violating the claimed inequality.
It therefore suffices to consider the case $d_1=2$. In this case, \Cref{Lemma:Recursive} and the ambient dimension bound give
\[
\begin{aligned}
\dimV_{\bfd}
&\leq
\dimV_{(2,2)}
+\dimV_{(2,d_2,d_3,d_4,d_5,d_6,1)}
-2\\
&\leq
4+\dim\Sym_{32}(\RR^2)-2
=4+33-2
=35<65.
\end{aligned}
\]
Thus $\bfd$ is not filling when $d_1\leq2$, and hence every filling architecture must satisfy~$d_1\geq3$.

\begingroup
\setlength{\abovedisplayskip}{0pt} \setlength{\belowdisplayskip}{0pt} \setlength{\abovedisplayshortskip}{0pt} \setlength{\belowdisplayshortskip}{0pt}
The other five inequalities follow in the same way:
\begin{itemize}[nosep]
\item If $d_2=2$, then
\[
\begin{aligned}
\dimV_{\bfd}
&\leq
\dimV_{(2,d_1,2)}
+\dimV_{(2,d_3,d_4,d_5,d_6,1)}
-2\\
&\leq
2\dim\Sym_2(\RR^2)
+\dim\Sym_{16}(\RR^2)
-2\\
&=6+17-2
=21<65.
\end{aligned}
\]

\item If $d_3=3$, then
\[
\begin{aligned}
\dimV_{\bfd}
&\leq
\dimV_{(2,d_1,d_2,3)}
+\dimV_{(3,d_4,d_5,d_6,1)}
-3\\
&\leq
3\dim\Sym_4(\RR^2)
+\dim\Sym_8(\RR^3)
-3\\
&=15+45-3
=57<65.
\end{aligned}
\]

\item If $d_4=3$, then
\[
\begin{aligned}
\dimV_{\bfd}
&\leq
\dimV_{(2,d_1,d_2,d_3,3)}
+\dimV_{(3,d_5,d_6,1)}
-3\\
&\leq
3\dim\Sym_8(\RR^2)
+\dim\Sym_4(\RR^3)
-3\\
&=27+15-3
=39<65.
\end{aligned}
\]

\item If $d_5=3$, then
\[
\begin{aligned}
\dimV_{\bfd}
&\leq
\dimV_{(2,d_1,d_2,d_3,d_4,3)}
+\dimV_{(3,d_6,1)}
-3\\
&\leq
3\dim\Sym_{16}(\RR^2)
+\dim\Sym_2(\RR^3)
-3\\
&=51+6-3
=54<65.
\end{aligned}
\]

\item If $d_6=1$, then
\[
\begin{aligned}
\dimV_{\bfd}
&\leq
\dimV_{(2,d_1,d_2,d_3,d_4,d_5,1)}
+\dimV_{(1,1)}
-1\\
&\leq
\dim\Sym_{32}(\RR^2)+1-1
=33<65.
\end{aligned}
\]
\end{itemize}
\endgroup

\medskip
\noindent Therefore, every filling architecture
$\bfd$ of depth $7$ satisfies
$(2,3,3,4,4,4,2,1)\leq\bfd$.

\smallskip
\noindent\emph{Reduction to fifteen regions.}
Applying \Cref{alg:randomized-image-dimension} to each of the thirteen
architectures of depth $7$ listed in \Cref{exhaustivelist} returns the
full ambient rank $65$. Hence all thirteen architectures are filling.

Removing the upward closures of the 13 listed architectures in \Cref{exhaustivelist} of depth $7$ from the set of architectures satisfying the six necessary inequalities above leaves the following fifteen maximal rectangular regions:
\[
\begingroup
\footnotesize
\renewcommand{\arraystretch}{1.17}
\setlength{\arraycolsep}{2.6pt}
\begin{array}{r@{\ }c@{\ }c@{\ }c@{\ }c@{\ }c@{\ }c@{\ }c@{\ }c@{\ }c@{\ }c@{\ }c}
\text{Region 1:}
& [3,\infty) & \times & [3,3] & \times & [4,\infty)
& \times & [4,4] & \times & [4,\infty) & \times & [2,\infty)
\\
\text{Region 2:}
& [3,\infty) & \times & [3,3] & \times & [4,4]
& \times & [4,5] & \times & [4,5] & \times & [2,\infty)
\\
\text{Region 3:}
& [3,\infty) & \times & [3,3] & \times & [4,\infty)
& \times & [4,5] & \times & [4,4] & \times & [2,\infty)
\\
\text{Region 4:}
& [3,\infty) & \times & [3,3] & \times & [4,\infty)
& \times & [4,5] & \times & [4,5] & \times & [2,3]
\\
\text{Region 5:}
& [3,\infty) & \times & [3,3] & \times & [4,5]
& \times & [4,6] & \times & [4,4] & \times & [2,3]
\\
\text{Region 6:}
& [3,\infty) & \times & [3,3] & \times & [4,\infty)
& \times & [4,6] & \times & [4,4] & \times & [2,2]
\\
\text{Region 7:}
& [3,\infty) & \times & [3,3] & \times & [4,\infty)
& \times & [4,5] & \times & [4,\infty) & \times & [2,2]
\\
\text{Region 8:}
& [3,\infty) & \times & [3,\infty) & \times & [4,4]
& \times & [4,4] & \times & [4,\infty) & \times & [2,\infty)
\\
\text{Region 9:}
& [3,\infty) & \times & [3,\infty) & \times & [4,4]
& \times & [4,5] & \times & [4,5] & \times & [2,3]
\\
\text{Region 10:}
& [3,\infty) & \times & [3,\infty) & \times & [4,4]
& \times & [4,\infty) & \times & [4,4] & \times & [2,\infty)
\\
\text{Region 11:}
& [3,\infty) & \times & [3,\infty) & \times & [4,4]
& \times & [4,5] & \times & [4,\infty) & \times & [2,2]
\\
\text{Region 12:}
& [3,\infty) & \times & [3,\infty) & \times & [4,\infty)
& \times & [4,4] & \times & [4,5] & \times & [2,\infty)
\\
\text{Region 13:}
& [3,\infty) & \times & [3,4] & \times & [4,5]
& \times & [4,5] & \times & [4,4] & \times & [2,3]
\\
\text{Region 14:}
& [3,\infty) & \times & [3,\infty) & \times & [4,\infty)
& \times & [4,5] & \times & [4,4] & \times & [2,2]
\\
\text{Region 15:}
& [3,\infty) & \times & [3,\infty) & \times & [4,\infty)
& \times & [4,4] & \times & [4,\infty) & \times & [2,3].
\end{array}
\endgroup
\]

\smallskip
\noindent\emph{Exclusion of the fifteen regions.}
For Region~$13$, we first establish the auxiliary bound
\[
\dimV_{(2,d_1,4,5)}\leq24.
\]
Let $q_1,\dots,q_4$ be the quadrics produced by the width-$4$ hidden
layer of $(2,d_1,4,5)$. Each of the five output quartics is a linear
combination of
$q_1^2,\dots,q_4^2$.
After choosing a basis of $\Sym_4(\RR^2)$, the coefficient matrix of
the five output quartics is a $5\times5$ matrix of rank at most $4$.
Thus the neurovariety is contained in the determinantal hypersurface of $5\times5$ matrices of rank at most $4$. 
Since this hypersurface has dimension $24$, we obtain $\dimV_{(2,d_1,4,5)}\leq24$.

The dimension upper bounds for the fifteen regions are collected in
\Cref{table:region-inequality-details}. They follow from repeated
applications of \Cref{Lemma:Recursive}, together with ambient- and
expected-dimension bounds. For Regions~$4$, $5$, $9$, and $13$, we
also use the dimension computations in
\Cref{lem:L7-dimension-counts}.

Every upper bound in the table is strictly smaller than the ambient dimension $65$. Hence every architecture in the fifteen regions is non-filling. Therefore, every filling architecture of depth $7$ lies in the upward closure of one of the thirteen listed architectures. Since the thirteen listed architectures are themselves filling, they are exactly the MFAs of depth $\depth=7$.
\end{proof}

Combining \Cref{lemma:complete-through-depth-six,lemma:complete-through-depth-seven}
proves that \Cref{exhaustivelist} is complete for $\depth\leq7$.

\begin{table}[p]
\centering
\caption{Dimension upper bounds for the fifteen regions. The final column records the width conditions used
in deriving the corresponding upper bound.}
\label{table:region-inequality-details}

\footnotesize
\renewcommand{\arraystretch}{1.22}
\setlength{\tabcolsep}{4pt}
\begin{tabular}{@{}>{\raggedright\arraybackslash}p{0.07\textwidth}
>{\raggedright\arraybackslash}p{0.7\textwidth}
>{\raggedright\arraybackslash}p{0.24\textwidth}@{}}
\toprule
Region & Dimension upper bound & Width conditions used \\
\midrule
\midrule
$1$ &
$\displaystyle\begin{aligned}[t]
\dimV_{\bfd}&\leq \dimV_{(2,d_1,3)}+\dimV_{(3,d_3,4)}+\dimV_{(4,d_5,d_6,1)}-3-4\\
&\leq 9+24+35-3-4=61
\end{aligned}$
& $\displaystyle\begin{aligned}[t]d_2&=3,\ d_4=4\end{aligned}$ \\
\midrule
$2$ &
$\displaystyle\begin{aligned}[t]
\dimV_{\bfd}&\leq \dimV_{(2,d_1,3)}+\dimV_{(3,4)}+\dimV_{(4,d_4,d_5,d_6,1)}-3-4\\
&\leq 9+12+35-3-4=49
\end{aligned}$
& $\displaystyle\begin{aligned}[t]d_2&=3,\ d_3=4,\\ 
d_4&\leq5,\ d_5\leq5\end{aligned}$ \\
\midrule
$3$ &
$\displaystyle\begin{aligned}[t]
\dimV_{\bfd}&\leq \dimV_{(2,d_1,3)}+\dimV_{(3,d_3,d_4)}+\dimV_{(d_4,4)}+\dimV_{(4,d_6,1)}-3-d_4-4\\
&\leq 9+6d_4+4d_4+10-3-d_4-4\leq57
\end{aligned}$
& $\displaystyle\begin{aligned}[t]d_2&=3,\ d_4\leq5,\ d_5=4\end{aligned}$ \\
\midrule
$4$ &
$\displaystyle\begin{aligned}[t]
\dimV_{\bfd}&\leq \dimV_{(2,d_1,3)}+\dimV_{(3,d_3,5)}+\dimV_{(5,5,3,1)}-3-5\\
&\leq 9+30+32-3-5=63
\end{aligned}$
& $\displaystyle\begin{aligned}[t]d_2&=3,\ d_4\leq5,\\
d_5&\leq5,\ d_6\leq3\end{aligned}$ \\
\midrule
$5$ &
$\displaystyle\begin{aligned}[t]
\dimV_{\bfd}&\leq \dimV_{(2,d_1,3)}+\dimV_{(3,5,6,4,3,1)}-3\\
&\leq 9+58-3=64
\end{aligned}$
& $\displaystyle\begin{aligned}[t]d_2&=3,\ d_3\leq5,\ d_4\leq6,\\ d_5&=4,\ d_6\leq3\end{aligned}$ \\
\midrule
$6$ &
$\displaystyle\begin{aligned}[t]
\dimV_{\bfd}&\leq \dimV_{(2,d_1,3)}+\dimV_{(3,d_3,d_4)}+\dimV_{(d_4,4)}+\dimV_{(4,2)}+\dimV_{(2,1)}-3-d_4-4-2\\
&\leq 9+6d_4+4d_4+8+2-3-d_4-4-2\leq64
\end{aligned}$
& $\displaystyle\begin{aligned}[t]d_2&=3,\ d_4\leq6,\\ d_5&=4,\ d_6=2\end{aligned}$ \\
\midrule
$7$ &
$\displaystyle\begin{aligned}[t]
\dimV_{\bfd}&\leq \dimV_{(2,d_1,3)}+\dimV_{(3,d_3,d_4)}+\dimV_{(d_4,d_5,2)}+\dimV_{(2,1)}-3-d_4-2\\
&\leq 9+6d_4+2\binom{d_4+1}{2}+2-3-d_4-2\leq56
\end{aligned}$
& $\displaystyle\begin{aligned}[t]d_2&=3,\ d_4\leq5,\ d_6=2\end{aligned}$ \\
\midrule
$8$ &
$\displaystyle\begin{aligned}[t]
\dimV_{\bfd}&\leq \dimV_{(2,d_1,d_2,4)}+\dimV_{(4,4)}+\dimV_{(4,d_5,d_6,1)}-4-4\\
&\leq20+16+35-4-4=63
\end{aligned}$
& $\displaystyle\begin{aligned}[t]d_3&=4,\ d_4=4\end{aligned}$ \\
\midrule
$9$ &
$\displaystyle\begin{aligned}[t]
\dimV_{\bfd}&\leq \dimV_{(2,d_1,d_2,4)}+\dimV_{(4,5,5,3,1)}-4\\
&\leq20+47-4=63
\end{aligned}$
& $\displaystyle\begin{aligned}[t]d_3&=4,\ d_4\leq5,\\ 
d_5&\leq5,\ d_6\leq3\end{aligned}$ \\
\midrule
$10$ &
$\displaystyle\begin{aligned}[t]
\dimV_{\bfd}&\leq \dimV_{(2,d_1,d_2,4)}+\dimV_{(4,d_4,4)}+\dimV_{(4,d_6,1)}-4-4\\
&\leq20+40+10-4-4=62
\end{aligned}$
& $\displaystyle\begin{aligned}[t]d_3&=4,\ d_5=4\end{aligned}$ \\
\midrule
$11$ &
$\displaystyle\begin{aligned}[t]
\dimV_{\bfd}&\leq \dimV_{(2,d_1,d_2,4)}+\dimV_{(4,d_4)}+\dimV_{(d_4,d_5,2)}+\dimV_{(2,1)}-4-d_4-2\\
&\leq20+4d_4+2\binom{d_4+1}{2}+2-4-d_4-2\leq61
\end{aligned}$
& $\displaystyle\begin{aligned}[t]d_3&=4,\ d_4\leq5,\ d_6=2\end{aligned}$ \\
\midrule
$12$ &
$\displaystyle\begin{aligned}[t]
\dimV_{\bfd}&\leq \dimV_{(2,d_1,d_2,d_3,4)}+\dimV_{(4,d_5)}+\dimV_{(d_5,d_6,1)}-4-d_5\\
&\leq36+4d_5+\binom{d_5+1}{2}-4-d_5\leq62
\end{aligned}$
& $\displaystyle\begin{aligned}[t]d_4&=4,\ d_5\leq5\end{aligned}$ \\
\midrule
$13$ &
$\displaystyle\begin{aligned}[t]
\dimV_{\bfd}&\leq \dimV_{(2,d_1,4,5)}+\dimV_{(5,5)}+\dimV_{(5,4,3,1)}-5-5\\
&\leq24+25+25-10=64
\end{aligned}$
& $\displaystyle\begin{aligned}[t]d_2&\leq4,\ d_3\leq5,\ d_4\leq5,\\ d_5&=4,\ d_6\leq3\end{aligned}$ \\
\midrule
$14$ &
$\displaystyle\begin{aligned}[t]
\dimV_{\bfd}&\leq \dimV_{(2,d_1,d_2,d_3,d_4)}+\dimV_{(d_4,4)}+\dimV_{(4,2)}+\dimV_{(2,1)}-d_4-4-2\\
&\leq9d_4+4d_4+8+2-d_4-4-2\leq64
\end{aligned}$
& $\displaystyle\begin{aligned}[t]d_4&\leq5,\ d_5=4,\\ d_6&=2\end{aligned}$ \\
\midrule
$15$ &
$\displaystyle\begin{aligned}[t]
\dimV_{\bfd}&\leq \dimV_{(2,d_1,d_2,d_3,4)}+\dimV_{(4,d_5,d_6)}+\dimV_{(d_6,1)}-4-d_6\\
&\leq36+10d_6+d_6-4-d_6\leq62
\end{aligned}$
& $\displaystyle\begin{aligned}[t]d_4&=4,\ d_6\leq3\end{aligned}$ \\
\bottomrule
\end{tabular}
\end{table}

\newpage
\noindent
\footnotesize {\bf Authors' addresses:}
\smallskip

\smallskip
\noindent Kevin Dao, University of Wisconsin--Madison, USA \hfill {\tt  ktdao@wisc.edu}\newline
\url{https://daokevin06.github.io/} 

\smallskip
\noindent Jose Israel Rodriguez, University of Wisconsin--Madison, USA \hfill {\tt  jose@math.wisc.edu}\newline
\url{https://sites.google.com/wisc.edu/jose/}

\end{document}